\documentclass[oneside,11pt]{article} 
\renewenvironment{abstract}
  {{\centering\large\bfseries Abstract\par}\vspace{0.7ex}%
    \bgroup
       \leftskip 20pt\rightskip 20pt\small\noindent\ignorespaces}%
  {\par\egroup\vskip 0.25ex}

\usepackage{microtype}
\usepackage{graphicx}

\usepackage{amsmath,amsfonts,amssymb,amsthm,commath}

\usepackage{algorithm,algorithmic}

\usepackage{booktabs,enumitem}
\usepackage{url}
\usepackage{xcolor}

\usepackage{hyperref} 

\usepackage{balance} 

\usepackage{pifont}

\theoremstyle{plain}
\newtheorem{theorem}{Theorem}
\newtheorem{lemma}[theorem]{Lemma}
\newtheorem{proposition}[theorem]{Proposition}

\theoremstyle{definition}
\newtheorem{definition}[theorem]{Definition}

\newtheorem{claim}[theorem]{Claim}
\newtheorem{assumption}{Assumption}

\newcommand{\defeq}{:=}

\renewcommand{\(}{\left(}
\renewcommand{\)}{\right)}

\newcommand{\poly}[1]{\operatorname{poly}\del{#1}}
\newcommand{\polylog}[1]{\operatorname{polylog}\del{#1}}
\newcommand{\sign}[1]{\operatorname{sign}\del{#1}}
\renewcommand{\Pr}{\operatorname{Pr}}
\newcommand{\EXP}{\operatorname{\mathbb{E}}}

\newcommand{\err}{\operatorname{err}}

\newcommand{\trans}{^{\top}}

\newcommand{\ind}[1]{\boldsymbol{1}_{\cbr{#1}}}

\newcommand{\R}{\mathbb{R}}
\newcommand{\Rd}{\mathbb{R}^d}

\newcommand{\calD}{\mathcal{D}}
\newcommand{\calX}{\mathcal{X}}
\newcommand{\calY}{\mathcal{Y}}

\newcommand{\calF}{\mathcal{F}}

\newcommand{\calR}{\mathcal{R}}

\newcommand{\twonorm}[1]{\left\lVert #1 \right\rVert_{2}}

\newcommand{\oraclex}{\mathrm{EX}_{\eta}^x(D, w^*)}
\newcommand{\oracley}{\mathrm{EX}^y}

\newcommand{\lambdamax}{\lambda_{\max}}

\newcommand{\mumax}{\mu_{\max}}

\newcommand{\nastyoraclexy}{\mathrm{EX}_{\eta}(D, w^*; N)}
\newcommand{\nastyoraclex}{\mathrm{EX}_{\eta}^x(D, w^*; N)}

\newcommand{\TC}{T_{\mathrm{C}}}
\newcommand{\TD}{T_{\mathrm{D}}}
\newcommand{\TE}{T_{\mathrm{E}}}

\newcommand{\hatTC}{\hat{T}_{\mathrm{C}}}
\newcommand{\hatTE}{\hat{T}_{\mathrm{E}}}

\newcommand{\ND}{N_{\mathrm{D}}}
\newcommand{\pD}{p^{}_{\mathrm{D}}}
\newcommand{\AC}{A_{\mathrm{C}}}
\newcommand{\AD}{A_{\mathrm{D}}}
\renewcommand{\AE}{A_{\mathrm{E}}}

\newcommand{\citep}[1]{\cite{#1}}
\newcommand{\citet}[1]{\cite{#1}}

\title{Sample-Optimal PAC Learning of Halfspaces with Malicious Noise}

\author{%
  Jie Shen \\
Stevens Institute of Technology\\
  \texttt{jie.shen@stevens.edu} 
}

\begin{document}

\maketitle

\begin{abstract}
We study efficient PAC learning of homogeneous halfspaces in $\mathbb{R}^d$ in the presence of malicious noise of Valiant~(1985). This is a challenging noise model and only until recently has near-optimal noise tolerance bound been established under the mild condition that the unlabeled data distribution is isotropic log-concave. However, it remains unsettled how to obtain the optimal sample complexity simultaneously. In this work, we present a new analysis for the algorithm of Awasthi~et~al.~(2017) and show that it essentially achieves the near-optimal sample complexity bound of $\tilde{O}(d)$, improving the best known result of $\tilde{O}(d^2)$. Our main ingredient is a novel incorporation of a matrix Chernoff-type inequality to bound the spectrum of an empirical covariance matrix for well-behaved distributions, in conjunction with a careful exploration of the localization schemes of Awasthi~et~al.~(2017). We further extend the algorithm and analysis to the more general and stronger nasty noise model of Bshouty~et~al.~(2002), showing that it is still possible to achieve near-optimal noise tolerance and sample complexity  in polynomial time under a mild relaxation of the noise model.
\end{abstract}

\section{Introduction}\label{sec:intro}

In this paper, we study computationally efficient PAC learning of homogeneous halfspaces~--~arguably one of the most important problems in machine learning~\cite{valiant1984theory}. In the absence of noise, the problem is well understood and can be efficiently solved by linear programming~\cite{maass1994fast} or the Perceptron~\cite{rosenblatt1958perceptron}. However, when the unlabeled data\footnote{We will also refer to unlabeled data as instances in this paper, and refer to labeled data as samples.} or the labels are corrupted, it becomes subtle to develop polynomial-time algorithms that are resilient to the noise~\cite{valiant1985learning,angluin1988learning,kearns1988learning,kearns1992toward}.

Generally speaking, a large body of existing works study the problem of learning halfspaces under {\em label} noise. This includes early works on random classification noise where the label of each instance is independently flipped with a fixed probability~\cite{blum1996polynomial}, a more general model termed Massart noise where the probability of flipping a given label may vary from instance to instance but is bounded away from $\frac12$~\cite{sloan1988types,massart2006risk}, the Tsybakov noise where the flipping probability can be arbitrarily close to $\frac12$ for a fraction of samples~\cite{tsybakov2004optimal}, and the much stronger adversarial (i.e. agnostic) noise where the adversary may choose an arbitrary joint distribution over the instance and label spaces~\cite{haussler1992decision,kearns1992toward,kalai2005agnostic,daniely2015ptas}. When only the labels are corrupted, significant progress towards establishing near-optimal performance guarantees has been witnessed in recent years; see, e.g.~\citet{awasthi2017power,diakonikolas2019distribution,diakonikolas2020near,diakonikolas2020polynomial,diakonikolas2020learning,zhang2020efficient,shen2020power}.

Compared to the fruitful set of positive results on efficient learning of halfspaces under label noise, less is known for the significantly more challenging regime where {\em both} instances and labels are corrupted. Specifically, one of such strong noise models that has played a crucial role in learning theory is the malicious noise model of~\citet{valiant1985learning,kearns1988learning}, defined as follows:

\begin{definition}[Malicious noise]\label{def:malicious}
Let $\calX = \Rd$ and $\calY = \{-1, 1\}$ be the instance and label space, respectively. Let $D$ be an unknown distribution over $\calX$, and $w^* \in \Rd$ be an unknown halfspace. Each time the learner requests a sample, with probability $1 - \eta$, the adversary draws $x$ from $D$ and returns the clean sample $(x, \sign{w^* \cdot x})$; with probability $\eta$, it may return an arbitrary pair $(x, y) \in \calX \times \calY$ called dirty sample. The parameter $\eta \in [0, \frac12)$ is termed noise rate.
\end{definition}

Notably, when the adversary is allowed to search for dirty samples, it is assumed to have unlimited computational power and can construct the sample based on the state of the learning algorithm and the history of its outputs. Since this is a much more demanding noise model (compared to label-only noise), even the achievability of optimal noise tolerance by efficient algorithms remained unsettled for decades. For example, the early work of \citet{kearns1988learning} presented a general analysis showing that even without any distributional assumptions, it is possible to tolerate the malicious noise at a rate of ${\Omega}(\epsilon/d)$, but a noise rate greater than $\frac{\epsilon}{1+\epsilon}$ cannot be tolerated, where $\epsilon \in (0, 1)$ is the target error rate given to the learner. The noise model was then broadly studied in the literature, though the learning algorithms may be inefficient; see e.g. \citet{schapire1992design,bshouty1998new,cesa1999sample}. Under different distributional assumptions, there are more positive results for efficient learning with malicious noise. In particular, when the distribution $D$ is uniform over the unit sphere, \citet{kalai2005agnostic} developed an efficient learning algorithm and obtained a noise tolerance $\Omega(\epsilon/d^{1/4})$, which was later improved to $\Omega(\epsilon^2/\log(d/\epsilon))$ in terms of the dependence on the dimension by \citet{klivans2009learning}. It is, however, well recognized that the uniform distribution is often restrictive in practice. As a remedy, \citet{klivans2009learning} also investigated the remarkably more general isotropic log-concave distributions~\cite{lovasz2007geometry}, and showed for the first time a noise tolerance of $\Omega(\epsilon^3/\log^2(d/\epsilon))$ under such mild condition. Unfortunately, owing to the strong power of the adversary, the barrier of achieving the information-theoretic limit of $\frac{\epsilon}{1+\epsilon}$ was not broken for many years (even under the uniform distribution).  Very recently, a near-optimal noise tolerance of the form $\Omega(\epsilon)$ was established by \citet{awasthi2017power} for isotropic log-concave distributions through a dedicated iterative localization technique, which stands for the state of the art.

In addition to the degree of noise tolerance, another yet important quantity that characterizes the performance of a learning algorithm is sample complexity. Unfortunately, it turns out that none of the prior works obtained near-optimal sample complexity and noise tolerance simultaneously under the mild condition that the (clean) instances are drawn from an isotropic log-concave distribution. In particular, \citet{awasthi2017power,shen2020attribute} obtained state-of-the-art noise tolerance but the sample complexity of \citet{awasthi2017power} reads as $\tilde{O}(d^3)$. \citet{shen2020attribute} considered learning of $s$-sparse halfspaces with malicious noise and showed through a refined analysis that a sample size of $\tilde{O}(s^2 \cdot \polylog{d})$ suffices; when specified to the non-sparse setting (which is the focus of this paper), it still leads to a suboptimal bound of $\tilde{O}(d^2)$. Prior to these two recent works, even a noise tolerance of the form $\Omega(\epsilon)$ was not established, nor an optimal sample complexity bound. On the other hand, it is worth mentioning that under the fairly restrictive uniform distribution over the unit ball, the analysis of \citet{awasthi2017power} does imply a near-optimal sample complexity bound of $\tilde{O}(d)$. In this regard, a natural question is: {\em can we design a computationally efficient algorithm that is able to tolerate $\Omega(\epsilon)$ malicious noise while enjoying the optimal sample complexity bound of $O(d)$ under isotropic log-concave distributions?}

In this paper, we answer the question in the affirmative. First, we formally describe our assumption on clean instances.

\begin{assumption}\label{as:x}
The distribution $D$ is isotropic log-concave over $\Rd$; namely, it has zero mean and unit covariance matrix, and the logarithm of its density function is concave.
\end{assumption}

Observe that the family of isotropic log-concave distributions covers prominent distributions such as Gaussian, exponential, and logistic distributions~\cite{lovasz2007geometry,vempala2010random}. In particular, general isotropic log-concave distributions are asymmetric in nature and the magnitude of the instances drawn from them is often dimension-dependent, making it nontrivial to extend results developed for the uniform distribution over the unit ball.

\subsection{Main results}

Recall that $D$ and $w^*$ are the underlying distribution and the correct halfspace as stated in Definition~\ref{def:malicious}, respectively. 
For any homogeneous halfspace $h_w: x \mapsto \sign{w\cdot x}$, let $\err_D(w) := \Pr_{x \sim D}(\sign{{w}\cdot x} \neq \sign{w^* \cdot x})$ be the error rate of $w$ with respect to $D$ and $w^*$. The following is our main result.

\begin{theorem}\label{thm:malicious-informal}
Consider the malicious noise model under Assumption~\ref{as:x}. There is an algorithm such that for any  target error rate $\epsilon \in (0, 1)$  and any failure probability $\delta \in (0, 1)$, if  $\eta \leq O(\epsilon)$, it outputs a halfspace $\tilde{w}$ satisfying $\err_D(\tilde{w}) \leq \epsilon$ with probability $1-\delta$. The running time is $\poly{d, \frac{1}{\epsilon}, \frac{1}{\delta}}$ and the sample complexity is $O\Big(\frac{d}{\epsilon}\cdot \polylog{d, \frac{1}{\epsilon}, \frac{1}{\delta}}\Big)$.
\end{theorem}

We highlight that this is the first result for efficient PAC learning of homogeneous halfspaces with both near-optimal malicious noise tolerance and sample complexity under isotropic log-concave distributions. On the algorithmic spectrum, we in fact show that the active learning algorithm proposed by \citet{awasthi2017power} inherently enjoys the announced properties and the noise tolerance bound in Theorem~\ref{thm:malicious-informal} directly inherits from their results. 
Regarding sample complexity, their original analysis made use of the pseudo-dimension from VC theory~\cite{anthony1999neural} to give an $\tilde{O}(d^3)$ sample complexity bound which is suboptimal. Even using a careful Rademacher complexity bound, we would only obtain an $\tilde{O}(d^2)$ bound. Our improvement comes from a reformulation of the objective function used by \citet{awasthi2017power} and a novel utilization of a matrix Chernoff-type inequality due to \citet{tropp2012user}, together with a careful exploration of the localization schemes of \citet{awasthi2017power}; see Section~\ref{sec:mal} for more details.

\subsection{Extension to the nasty noise model}

We also consider learning of homogeneous halfspaces with nasty noise of \citet{bshouty2002pac}, which is a strict generalization and is stronger than the malicious noise model.

\begin{definition}[Nasty noise]\label{def:nasty}
The learner specifies the total number of needed samples $N$. The adversary takes as input $N$, draws such many independent instances from $D$, and labels them correctly according to $w^*$. Then it may replace an arbitrary $\eta$ fraction of them with arbitrary samples in $\calX \times \calY$. The corrupted sample set is returned to the learner.
\end{definition}

Observe that when $N=1$, it reduces to the malicious noise model. The additional power of the nasty  adversary is that when $N > 1$, it may inspect {\em all} the clean samples, and then decides which of them will be replaced, while in the malicious noise model it can only inject dirty samples (when it is permitted). Note that such extra power of erasing clean instances may modify the marginal distribution of the clean instances returned to the learner, which is one of the technical barriers that we have to carefully address.

For the problem of learning homogeneous halfspaces with nasty noise, we show that the algorithm of \citet{awasthi2017power} still works well (hence, our contribution is a new  analysis). We have the following performance guarantee.

\begin{theorem}\label{thm:nasty-informal}
Consider the nasty noise model under Assumption~\ref{as:x}. There is an algorithm such that for any  target error rate $\epsilon \in (0, 1)$ and any failure probability $\delta \in (0, 1)$, if $\eta \leq O(\epsilon)$ and the learner is allowed to call the nasty oracle $O(\log\frac{1}{\epsilon})$ times, it outputs a halfspace $\tilde{w}$ satisfying $\err_D(\tilde{w}) \leq \epsilon$ with probability $1-\delta$. The running time is $\poly{d, \frac{1}{\epsilon}, \frac{1}{\delta}}$ and the sample complexity is $O\Big(\frac{d}{\epsilon}\cdot \polylog{d, \frac{1}{\epsilon}, \frac{1}{\delta}}\Big)$.
\end{theorem}

Since the malicious noise is a special case of the nasty noise, the information-theoretic limit of the noise tolerance established in \citet{kearns1988learning}, i.e. $\frac{\epsilon}{1+\epsilon}$, also applies to the nasty noise. In other words, the noise tolerance in the above theorem is near-optimal as well.

Another salient feature coming with the learning algorithm we consider, i.e. the one developed in \citet{awasthi2017power}, is label efficiency; that is, the label complexity of the algorithm is $O\Big(d \cdot \polylog{d, \frac{1}{\epsilon}, \frac{1}{\delta}}\Big)$ which has an exponential improvement on the dependence of $\frac{1}{\epsilon}$; see Appendix~\ref{sec:app:main-proof} for the proof. To the best of our knowledge, this is also the first label-efficient algorithm that tolerates the nasty noise.

We remark, however, that in many prior works, the learner typically makes a one-time call throughout the learning process to gather all the labeled instances~\cite{bshouty2002pac,diakonikolas2018learning}. Since we will study an algorithm that proceeds in multiple phases and draws samples adaptively, we consider a natural relaxation which allows the learner to make a one-time call per phase, with a total number of calls being $O(\log\frac{1}{\epsilon})$.\footnote{The crucial difference between the algorithm we consider and prior passive learning algorithms lies in the number of rounds that the learner communicates with the adversary.} Therefore, our results under the nasty noise model are {\em not} strictly comparable to prior results such as \citet{diakonikolas2018learning}. It remains open of how to design an efficient algorithm which gathers all samples in one batch while still enjoying  near-optimal nasty noise tolerance and sample complexity simultaneously.

\subsection{Related works}

The malicious and nasty noise models are strong contamination models for the problem of robustly learning Boolean functions. It turns out that most prior works on learning of halfspaces concentrated on obtaining optimal noise tolerance while not pursuing the $O(d)$ sample complexity, in that the former problem alone is already quite challenging~\cite{kearns1988learning,awasthi2017power}. \citet{diakonikolas2018learning} considered learning of more general concept classes, e.g. low-degree polynomial threshold functions and intersections of halfspaces, and showed that the underlying concept can be efficiently learned with $\tilde{O}(d^{\gamma})$ sample complexity for some unspecified constant $\gamma > 1$. When adapted to our setting (i.e. learning homogeneous halfspaces under isotropic log-concave marginal distributions), Theorem~1.5 of \citet{diakonikolas2018learning} only gives noise tolerance $\eta \leq O(\epsilon^{\gamma'})$ for some constant $\gamma' > 1$ which is suboptimal.

Recent works such as \citet{diakonikolas2016robust,lai2016agnostic} studied mean estimation under a nasty-type model where in addition to returning dirty instances, the adversary has also the power of eliminating a few clean instances. The key technique of robust mean estimation is to use the spectral norm of the empirical covariance matrix to detect dirty instances, and a sample complexity bound of $\tilde{O}(d)$ was obtained, typically under Gaussian distributions rather than the more general isotropic log-concave distributions. More recently, such technique was extensively investigated for a variety of problems such as clustering and linear regression; we refer the readers to the comprehensive survey of~\citet{diakonikolas2019recent}. From a high level, the idea of certifying clean instances with a small second order moment roots in a much earlier work by \citet{blum1996polynomial}, and was then serving as a crucial component in learning halfspaces with malicious noise \cite{klivans2009learning}. We note, however, that near-optimal sample complexity for learning halfspaces under  isotropic log-concave marginal distributions  is not implied by  these results. 




\vspace{0.1in}

\noindent{\bfseries Notations.} \ 
For a unit vector $u \in \Rd$ and a positive scalar $b$, we will frequently use $X_{u, b}$ to denote the band $\{x \in \Rd: \abs{u \cdot x} \leq b\}$. Let $T$ be a set of unlabeled data. We will use $\hat{T}$ to denote its labeled set, i.e. $\hat{T} = \{ (x, y_x): x \in T\}$ where $y_x$ is the label that the adversary is committed to. We write $\tilde{O}(f) := O( f \cdot \polylog f)$. The letters $c$ and $C$, and their subscript variants such as $c_1$, $C_1$, are reserved for specific absolute constants; see Appendix~\ref{sec:app:constants}.

\vspace{0.1in}
\noindent{\bfseries Roadmap.} \ 
In Section~\ref{sec:mal}, we briefly describe the algorithm of \citet{awasthi2017power}, followed by a refined theoretical analysis on the sample complexity. In Section~\ref{sec:nasty}, we extend the algorithm and analysis to the nasty noise model. We conclude this paper in Section~\ref{sec:conclusion}, and defer all the proof details to the appendix.

\section{Learning with Malicious Noise}\label{sec:mal}

We elaborate on our analytic tools used to obtain the near-optimal sample complexity bound in this section. Since we will give a new analysis for the algorithm developed by \citet{awasthi2017power}, we first briefly introduce their main mechanisms; readers are referred to their original work for more detailed technical descriptions.

To improve readability, throughout this section, we will always implicitly assume that Assumption~\ref{as:x} is satisfied.

\subsection{The approach of \citet{awasthi2017power}}\label{subsec:awasthi}

The malicious-noise-tolerant algorithm, i.e. Algorithm~2 of \citet{awasthi2017power}, is built upon the celebrated margin-based active learning framework of \citet{balcan2007margin}. For convenience, we record it in Algorithm~\ref{alg:main} with a minor simplification (to be clarified). At a high level, it proceeds in $K = O(\log\frac{1}{\epsilon})$ phases, where the key idea is to find in each phase an empirical minimizer of a certain hinge loss that is a good proxy of the loss on clean samples drawn from a localized instance space. The margin-based framework will then assert that this suffices for PAC learnability. To this end, in each phase it has three major steps: rejection sampling, soft outlier removal, and hinge loss minimization.

Let $X_{u, b} := \{x \in \Rd: \abs{u \cdot x} \leq b\}$ for some given unit vector $u$ and certain scalar $b \in [\epsilon, O(1)]$ where $\epsilon \in (0, 1)$ is the given target error rate; in the notation of Algorithm~\ref{alg:main}, $u$ should be thought of as $w_{k-1}$ and $ b= b_k$. Let $D_{u, b}$ be the distribution $D$ conditioned on the event that $x \in X_{u, b}$. During rejection sampling, the learner calls the adversary $\oraclex$ to collect a set $T$ of unlabeled data lying in the band $X_{u, b}$.\footnote{The algorithm of \citet{awasthi2017power} is active in nature. Thus, the adversary initially hides the label and only returns the instance; the learner must make a separate call to reveal the label.} Since the set $T$ is corrupted by the adversary, the goal of soft outlier removal is to find proper weights for all instances in $T$ such that the reweighted hinge loss over $T$ is almost equal to the one evaluated on clean samples. Based on the detection results, during hinge loss minimization, the learner makes an additional call to the oracle $\oracley$ to reveal the labels and finds an empirical minimizer of the reweighted hinge loss: 
\begin{equation*}
\ell_{\tau}(w; p \circ \hat{T}) := \frac{1}{\abs{T}} \sum_{ (x, y) \in \hat{T}} p(x) \cdot \max\Big\{ 0, 1 - \frac{1}{\tau} y w \cdot x \Big\}.
\end{equation*}
We remark that the sample complexity at phase $k$ refers to the number of calls to $\oraclex$, and the label complexity refers to that of $\oracley$.

It is known from standard margin-based active learning results that if the soft outlier removal step finds good weights in all the phases, then the final output of Algorithm~\ref{alg:main} will have small error rate~\cite{balcan2007margin,awasthi2017power}. Therefore, most of our discussions will be dedicated to this crucial step. Note that the sample complexity refers to the total number of calls to $\oraclex$ (which happens during rejection sampling) and the label complexity refers to that of $\oracley$ (which happens during loss minimization).

\begin{algorithm}[t]
\caption{Efficient and Sample-Optimal Algorithm Tolerating Malicious Noise}
\label{alg:main}
\begin{algorithmic}[1]
\REQUIRE Error rate $\epsilon$, failure probability $\delta$, instance generation oracle $\oraclex$, label revealing oracle $\oracley$.
\ENSURE Halfspace $\tilde{w}$ with  $\err_D(\tilde{w}) \leq \epsilon$ with probability $1-\delta$.
\STATE Initialize $w_0$ as the zero vector in $\Rd$.
\STATE $K \gets O(\log\frac{1}{\epsilon})$.
\FOR{phases $k = 1, 2, \dots, K$}
\STATE Clear the working set ${T}$.
\STATE $b_k \gets \Theta(2^{-k})$, $r_k \gets \Theta(2^{-k})$, $\tau_k \gets \Theta(2^{-k})$.
\STATE  Call $\oraclex$ for $N_k$ times to form instance set $A$. If $k = 1$, ${T} \gets A$; otherwise,  $T \gets \{x \in A: \abs{w_{k-1} \cdot x} \leq b_k\}$.

\STATE Apply Algorithm~\ref{alg:reweight} to $T$ with $u \leftarrow w_{k-1}$, $b \leftarrow b_k$, $r \leftarrow r_k$, $\xi \leftarrow \frac12- \Theta(1)$, $c \leftarrow 2 C_2$, and let $q = \cbr{q(x)}_{x \in T}$ be the returned function. Normalize $q$ to form a probability distribution $p$ over $T$.

\STATE $W_k \gets \{w: \twonorm{w} \leq 1, \twonorm{w - w_{k-1}} \leq r_k\}$, $\hat{T} \gets $ call $\oracley$ to reveal the labels of $T$. Find $v_k \in W_k$ with
\begin{equation*}
\ell_{\tau_k}(v_k; p \circ \hat{T}) \leq \min_{w \in W_k} \ell_{\tau_k}(w; p \circ \hat{T}) + O(1).
\end{equation*}
\STATE $w_k \leftarrow \frac{v_k}{\twonorm{v_k}}$.

\ENDFOR

\STATE {\bfseries return} $\tilde{w} \gets w_K$.
\end{algorithmic}
\end{algorithm}

\begin{algorithm}[t]
\caption{Localized Soft Outlier Removal}
\label{alg:reweight}
\begin{algorithmic}[1]
\REQUIRE{Reference unit vector $u$, band width $b>0$, radius $r = \Theta(b)$, empirical noise rate $\xi \in [0, 1/2)$, absolute constant $c > 0$, a set $T$ of instances  drawn from $D_{u, b}$.}
\ENSURE{A function $q: T \rightarrow [0, 1]$.}
\STATE \label{step:W} Let  $W = \big\{ w \in \Rd: \twonorm{w} \leq 1,\ \twonorm{w-u} \leq r \big\}$.

\STATE Find a function $q: T \rightarrow [0, 1]$ satisfying the following:
\begin{enumerate}
\item \label{item:alg2:0-1}  for all $x \in T, 0 \leq q(x) \leq 1$;

\item \label{item:alg2:err} $\sum_{x \in T} q(x) \geq (1 - \xi)\abs{T}$;

\item \label{item:alg2:var} $\sup_{w \in W} \frac{1}{\abs{T}} \sum_{x \in T} q(x)  (w \cdot x)^2 \leq c (b^2 + r^2)$.

\end{enumerate}

\STATE {\bfseries return} $q$.
\end{algorithmic}
\end{algorithm}

Decompose $T = \TC \cup \TD$ where $\TC$ denotes the set of clean instances in $T$ and $\TD$ for the dirty instances. The key algorithmic insight of \citet{awasthi2017power} is that in order to guarantee the success of soft outlier removal, i.e. Algorithm~\ref{alg:reweight} finds a feasible function $q: T \rightarrow [0, 1]$ in polynomial time, it is equivalent for the following to hold for some absolute constant $c > 0$:
\begin{equation}\label{eq:var}
\sup_{w \in W} \frac{1}{\abs{\TC}} \sum_{x \in \TC} (w \cdot x)^2 \leq c (b^2 + r^2),
\end{equation}
where
\begin{equation}\label{eq:W}
W := \big\{ w \in \Rd: \twonorm{w} \leq 1,\ \twonorm{w-u} \leq r \big\}.
\end{equation}

In order to prove \eqref{eq:var}, \citet{awasthi2017power} showed the following useful result.

\begin{lemma}\label{lem:E[wx^2]}
There is an absolute constant ${C}_2 \geq 1$ such that
\begin{equation*}
\sup_{w: \twonorm{w-u} \leq r} \EXP_{x \sim D_{u, b}}\sbr[1]{(w\cdot x)^2} \leq {C}_2 (b^2 + r^2).
\end{equation*}
\end{lemma}

On the other hand, \citet{anthony1999neural} proved that with high probability,
\begin{equation}\label{eq:VC}
\sup_{w \in W} \abs[3]{ \frac{1}{\abs{\TC}} \sum_{x \in \TC} (w \cdot x)^2 - \EXP_{x \sim D_{u, b}}\sbr[1]{(w\cdot x)^2} } \leq \alpha
\end{equation}
provided $\abs{\TC} = O\del[1]{\frac{(\rho^+ - \rho^-)^2}{\alpha^2} d}$ where $\rho^- := \inf_{w \in W} (w \cdot x)^2$ and $\rho^+ := \sup_{w \in W} (w \cdot x)^2$. 

Hence, \citet{awasthi2017power} combined Lemma~\ref{lem:E[wx^2]} and Eq.~\eqref{eq:VC} with $\alpha = {C}_2 (b^2 + r^2)$ and showed that \eqref{eq:var} holds with high probability if $\abs{\TC} = O\del[1]{\frac{(\rho^+ - \rho^-)^2}{\alpha^2} d}$. If the unlabeled data distribution $D$ were uniform over the unit sphere, then this bound would read as $O(d/\alpha^2)$ which has optimal dependence on $d$. However, since we are considering the significantly more general family of log-concave distributions, this bound becomes suboptimal. 
\begin{lemma}\label{lem:rho}
Consider $x \sim D_{u, b}$. Then with probability $1-\delta$, $(\rho^+ - \rho^-)^2 \leq O(d^2 \cdot b^4 \log^4\frac{1}{b\delta})$.
\end{lemma}

Therefore, the VC theory only leads to a suboptimal sample size $\abs{\TC} = {O}(d^3)$, which implies that the number of calls to the instance generation oracle must be $O(d^3)$.

We note that while Rademacher complexity may sometimes offer improved sample complexity as illustrated by \citet{zhang2018efficient,shen2020attribute}, for our problem it only gives suboptimal guarantee of $\abs{\TC} = \tilde{O}(d^2)$; see Appendix~\ref{sec:app:proof-mal}. Since the trouble roots in the suboptimal concentration bound of \eqref{eq:VC} which only involves quadratic functions, one may also wants to apply the well-known Hanson-Wright inequality~\cite{rudelson2013hanson} for better bound. The main barrier to apply it is that this inequality requires a sub-gaussian tail for the random vectors while that of log-concave distributions behaves as sub-exponential (see Part~\ref{item:ilc:tail} of Lemma~\ref{lem:logconcave}). 

\subsection{Our results and techniques}

In contrast to the quadratic dependence on the dimension $d$, we show that \eqref{eq:var} holds as soon as $\abs{\TC} = \tilde{O}(d)$.

\begin{theorem}\label{thm:var}
With probability $1-\delta$, Eq.~\eqref{eq:var} holds if $\ \abs{\TC} \geq d \cdot \polylog{d, \frac{1}{b}, \frac{1}{\delta}}$.
\end{theorem}

{Our technical novelty to show Theorem~\ref{thm:var} is to move away from  uniform concentration inequalities used by prior works.} Rather, we reformulate the objective function of \eqref{eq:var} which naturally leads to bounding the spectrum of a sum of random matrices. We then crucially explore the power of the localization scheme for the instance and concept spaces as used in Algorithm~\ref{alg:main}, and show that such spectrum norm acts as a constant over the phases, leading to the announced sample complexity. 

First of all, we use the basic fact that  for any $a_1, a_2 \in \R$, $(a_1 + a_2)^2 \leq 2(a_1^2 + a_2^2)$, and obtain  that
\begin{equation*}
\sup_{w \in W}\sum_{x \in \TC} (w \cdot x)^2 \leq \sup_{w \in W}\sum_{x \in \TC} \del{(w - u) \cdot x}^2 + \sum_{x \in \TC} (u \cdot x)^2.
\end{equation*}
Recall that in view of rejection sampling (namely localization in the instance space), for all $x \in \TC$, it was drawn from $D$ conditioned on the event $\abs{u \cdot x} \leq b$, implying  $(u\cdot x)^2 \leq b^2$ with certainty. Hence, it remains to upper bound $\sup_{w \in W} \del{(w - u) \cdot x}^2$. By the definition of $W$ in \eqref{eq:W}, we know that $w - u \in r \cdot V$ where $V := \{ v: \twonorm{v} \leq 1 \}$. It thus follows that
\begin{equation*}
\sup_{w \in W} \del{(w - u) \cdot x}^2 \leq r^2 \sup_{v \in V} (v \cdot x)^2 = r^2 \sup_{v \in V} v\trans (xx\trans ) v.
\end{equation*}
Putting all pieces together, we have that the left-hand side of \eqref{eq:var} can be upper bounded as follows:
\begin{equation}\label{eq:var-surrogate}
\sup_{w \in W} \frac{1}{\abs{\TC}} \sum_{x \in \TC} (w \cdot x)^2 \leq r^2 \sup_{v \in V} v\trans M v + b^2,
\end{equation}
where $M = \del[2]{ \frac{1}{\abs{\TC}} \sum_{x \in \TC} xx\trans }$. Observe that $v\trans M v$ corresponds to an eigenvalue of the matrix $M$. This motivates the consideration of the spectrum norm of the random matrix $M$, and is exactly where we need the matrix Chernoff bound of \citet{tropp2012user}.

\begin{lemma}[Matrix Chernoff inequality]\label{lem:matrix-chernoff}
Consider a finite sequence $\{ M_i \}_{i=1}^n$ of independent, random, self-adjoint matrices with dimension $d$. Assume that each random matrix satisfies $M_i \succeq 0$ and $\lambdamax(M_i) \leq \Lambda$ almost surely where $\lambdamax(\cdot)$ denotes the maximum eigenvalue. Define $\mumax := \lambdamax(\sum_{i=1}^{n} \EXP[M_i])$. Then for all $\alpha \geq 0$, with probability at least $1 - d \cdot \sbr{ \frac{e^{\alpha}}{(1+\alpha)^{1+\alpha}} }^{ \frac{\mumax}{\Lambda}}$,
\begin{equation*}
\lambdamax\del[2]{\sum_{i=1}^{n} M_i}  \leq (1+\alpha) \mumax.
\end{equation*}
\end{lemma}

To apply the above lemma, we will set $M_i = x_i x_i\trans$ for each $x_i \in \TC$. We also establish the following result to estimate the two important quantities $\Lambda$ and $\mumax$, where the proof crucially explores the localization scheme in the concept and instance spaces.

\begin{lemma}\label{lem:spectrum}
Suppose $x$ is randomly drawn from $D_{u, b}$. Then 
\begin{equation*}
\lambdamax\del[2]{ \EXP\sbr[1]{ x x\trans } } \leq \frac{4 C_2 (b^2 + r^2)}{r^2}.
\end{equation*}
In addition, with probability $1-\delta$,
\begin{equation*}
\lambdamax\del[1]{ xx\trans } \leq K_1 \cdot d  \log^2\frac{1}{b \delta}
\end{equation*}
for some constant $K_1 > 0$.
\end{lemma}

By setting $\alpha = 1$ in Lemma~\ref{lem:matrix-chernoff} and incorporating the results in Lemma~\ref{lem:spectrum}, we have the following:
\begin{proposition}\label{prop:var}
Let $\TC$ be a set of i.i.d. instances drawn from $D_{u, b}$. If $\abs{\TC} \geq d \cdot \polylog{d, \frac{1}{b}, \frac{1}{\delta}}$, then with probability $1-\delta$, $\lambdamax(M) \leq O(1)$. 
\end{proposition}

Now we are in the position to prove Theorem~\ref{thm:var}.
\begin{proof}[Proof of Theorem~\ref{thm:var}]
In fact, by \eqref{eq:var-surrogate} and Proposition~\ref{prop:var} we immediately have
\begin{equation*}
\sup_{w \in W} \frac{1}{\abs{\TC}} \sum_{x \in \TC} (w \cdot x)^2 \leq O(r^2) + b^2 \leq O(r^2 + b^2),
\end{equation*}
which is the desired result.
\end{proof}

Next, we need to translate the bound of $\abs{\TC}$ to that of the number of calls to $\oraclex$. To do so, we give a sufficient condition on the number of calls to $\oraclex$ under which, there are as many instances in $\TC$ as required in Theorem~\ref{thm:var}. This has been set out in \citet{awasthi2017power} where the primary observation is that under Assumption~\ref{as:x}, the probability mass of the band $X_{u, b}$ is $\Theta(b)$. Hence, by calling $\oraclex$ for $O(n/b)$ times it is guaranteed to gather $n$ instances to form $T$. Also, it is possible to show that the empirical noise rate within $T$ will be $O(\xi)$ where $\xi \in [0, 1/2)$ is a small constant, implying that $\abs{\TC} \geq \frac{1}{2}n$. By backward induction, the sample complexity in each phase is $\tilde{O}(d/b)$.

Formally, we have the following two lemmas.

\begin{lemma}\label{lem:N}
Assume $\eta < \frac{1}{2}$. By making a number of $N = O\del[2]{\frac{1}{ b}\del[1]{ n + \log\frac{1}{\delta} }}$ calls to $\oraclex$, we will obtain $n$ instances to form $T$ with probability $1- \delta$.
\end{lemma}

\begin{lemma}\label{lem:pruning}
Assume $\eta \leq c_5 \epsilon$ for small constant $c_5 > 0$. If $\abs{T} \geq 24\ln\frac{1}{\delta}$, then with probability $1 - {\delta}$, $\abs{\TC} \geq \frac{3}{4} \abs{T}$.
\end{lemma}

\vspace{0.1in}
\noindent{\bfseries Algorithmic simplification.} \
The last ingredient of Algorithm~\ref{alg:main} is hinge loss minimization. A slight improvement in light of our new sample complexity bound is that there is no need to perform random sampling of the instances in $T$ as in \citet{awasthi2017power}. This is because in their original analysis, $\abs{T} = \tilde{O}(d^3)$ and for the sake of optimizing label complexity without sacrificing the error rate, random sampling of a subset of size $\tilde{O}(d)$ is an elegant approach. In contrast, we already have shown that the size of $T$ itself is $\tilde{O}(d)$, which can be labeled entirely by querying $\oracley$.

We are now in the position to prove Theorem~\ref{thm:malicious-informal}. We note that the key difference from the analysis in \citet{awasthi2017power} is how we show that the $\tilde{O}(d)$ sample size suffices for soft outlier removal. We will therefore highlight this distinction in our proof. For the full and detailed proof, readers can either refer to their original paper or to our full proof of Theorem~\ref{thm:nasty-informal} in Section~\ref{sec:nasty} since the malicious noise is a special case of the nasty noise.

\begin{proof}[Proof Sketch of Theorem~\ref{thm:malicious-informal}]
Consider phase $k \leq K$. Combining Lemma~\ref{lem:N} and Lemma~\ref{lem:pruning}, we know that it suffices to call $\oraclex$ for $\tilde{O}(d/b_k)$ times to ensure the carnality of $\TC$ satisfies the condition in Theorem~\ref{thm:var}; hence the soft outlier removal step succeeds. On the other side, as shown in Proposition~11 of \citet{shen2020attribute}, such sample complexity bound also suffices to guarantee the uniform concentration of hinge loss. Since $b_k \geq \epsilon$ for all $k \leq K$, the per-phase sample complexity is $\tilde{O}(d/b_k) \leq \tilde{O}(d/\epsilon)$. Recall that there are a total of $O(\log\frac{1}{\epsilon})$ phases. Hence, the overall sample complexity is $\tilde{O}(d/\epsilon) \cdot \log\frac{1}{\epsilon} = \tilde{O}(d/\epsilon)$. 

The analyses of failure probability, error rate, and computational complexity are standard; see, e.g. \citet{awasthi2017power,shen2020attribute}.
\end{proof}

%

\section{Learning with Nasty Noise}\label{sec:nasty}

In this section we show that  the algorithm and analysis of \citet{awasthi2017power} can be modified to tolerate the nasty noise of \citet{bshouty2002pac}, with near-optimal noise tolerance and sample complexity.


Although under the nasty noise, we can still decompose $T = \TC \cup \TD$ where $\TC$ is the set of clean instances in $X_{u, b}$, the main technical challenge to  generalize the results of the preceding section to the nasty noise model is that the instances in $\TC$ may no longer be i.i.d. draws from $D_{u, b}$ due to the extra operation of erasing clean instances. Therefore, many crucial results such as Proposition~\ref{prop:var} do not hold, making it subtle to characterize the performance of soft outlier removal. Denote by $\TE$ the clean instances residing in $X_{u, b}$ but were replaced with dirty instances by the adversary. What we can say is that $\TC \cup \TE$ are i.i.d. draws from $D_{u, b}$, though $\TE$ is not accessible to the learner. Our main technical insight to handle the nasty noise is that if the nasty noise rate $\eta$ is small, and we perform instance localization (i.e. rejection sampling) as in Algorithm~\ref{alg:main}, then it is still possible to construct an {\em extended} empirical distribution over $T \cup \TE$ through soft outlier removal, under which the reweighted hinge loss on $T$ is almost a good proxy to the hinge loss on the {\em original} clean instances $\TC \cup \TE$. We show that to do so, it suffices to have $\abs{T} = \tilde{O}(d)$. Then we can reuse the analysis in Section~\ref{sec:mal} to show that the margin-based active learning algorithm PAC learns the underlying halfspace in polynomial time using $\tilde{O}(d/\epsilon)$ samples.

First of all, we formally introduce the problem setup with a few useful notations that will frequently be used in our subsequent analysis.

\vspace{0.1in}
\noindent{\bfseries Passive learning under nasty noise.} \ 
In the passive learning model, the sample generation oracle $\nastyoraclexy$ takes as input a sample size $N$ requested by the learner, draws $N$ i.i.d. instances $x_1, \dots, x_N$ from $D$ and labels them correctly, i.e. each $y_i = \sign{w^* \cdot x_i}$, which forms the labeled clean instance set $\hat{A}' = \{(x_1, y_1), \dots, (x_N, y_N)\}$. It then chooses any $\ND = \eta N$  samples in $\hat{A}'$, and replaces them with arbitrary pairs in $\calX \times \calY$. This corrupted sample set, denoted by $\hat{A}$, is returned to the learner. If we pass the parameter $N=1$ to $\nastyoraclexy$ and repeatedly call it, then the problem reduces to learning with malicious noise. However, if we must pass a large parameter $N$, the adversary under nasty noise is more powerful than that under malicious noise: it can inspect all the clean samples and decide which of them to corrupt. In passive learning, $\nastyoraclexy$ is often called only once with $N$ being the total number of samples needed by the learner~\cite{bshouty2002pac,diakonikolas2018learning}.


\vspace{0.1in}
\noindent{\bfseries Active learning under nasty noise.} \ 
In the active learning setting, the sample generation process remains unchanged. However, instead of having direct access to $\nastyoraclexy$ which returns the labeled set, the learner calls $\nastyoraclex$ to obtain the unlabeled corrupted instance set $A$ (i.e. $A$ is obtained by removing all the labels in $\hat{A}$). It can then decide to reveal  the labels for some of the instances in $A$ by calling $\oracley$. The sample complexity refers to the size of $A$, and the label complexity refers to the number of calls to $\oracley$.

\subsection{Algorithm}
The nasty-noise-tolerant algorithm is given in Algorithm~\ref{alg:nasty}, which is almost the same as Algorithm~\ref{alg:main}, with the major difference that the learner passes a parameter $N$ to  $\nastyoraclex$ at the beginning of each phase and only keeps those lying in a band to form the actually used instance set $T$. Then, during hinge loss minimization, the label revealing oracle $\oracley$ is called to reveal the labels of all instances in $T$. Since the size of $T$ is no greater than $N$, the number of unlabeled samples, such active learning scheme reduces the labeling cost. 

Observe that we make a total of $K$ calls to $\nastyoraclex$, which is different from the passive learning algorithms of \citet{bshouty2002pac,diakonikolas2018learning} in that they  call the oracle only once throughout learning; thus the results here are not strictly comparable to theirs but are still more general than what we have presented in Section~\ref{sec:mal}.

\begin{algorithm}[t]
\caption{Efficient and Sample-Optimal Algorithm Tolerating Nasty Noise}
\label{alg:nasty}
\begin{algorithmic}[1]
\REQUIRE Error rate $\epsilon$, failure probability $\delta$, instance generation oracle $\nastyoraclex$, label revealing oracle $\oracley$.
\ENSURE Halfspace $\tilde{w}$ with  $\err_D(\tilde{w}) \leq \epsilon$ with probability $1-\delta$.
\STATE Initialize $w_0$ as the zero vector in $\Rd$.
\STATE $K \gets O(\log\frac{1}{\epsilon})$.
\FOR{phases $k = 1, 2, \dots, K$}
\STATE Clear the working set ${T}$.
\STATE $b_k \gets \Theta(2^{-k})$, $r_k \gets \Theta(2^{-k})$, $\tau_k \gets \Theta(2^{-k})$.
\STATE  Call $\nastyoraclex$ with $N = N_k$ to form instance set $A$. If $k = 1$, ${T} \gets A$; otherwise,  $T \gets \{x \in A: \abs{w_{k-1} \cdot x} \leq b_k\}$.

{\STATE  Apply Algorithm~\ref{alg:reweight} to $T$ with $u \leftarrow w_{k-1}$, $b \leftarrow b_k$, $r \leftarrow r_k$, $\xi \leftarrow \xi_k$, $c \leftarrow 4 C_2$, and let $q = \cbr{q(x)}_{x \in T}$ be the returned function. Normalize $q$ to form a probability distribution $p$ over $T$.\label{step:soft}}

\STATE $W_k \gets \{w: \twonorm{w} \leq 1, \twonorm{w - w_{k-1}} \leq r_k\}$, $\hat{T} \gets $ call $\oracley$ to reveal the labels of $T$. Find $v_k \in W_k$ with
\begin{equation*}
\ell_{\tau_k}(v_k; p \circ \hat{T}) \leq \min_{w \in W_k} \ell_{\tau_k}(w; p \circ \hat{T}) + \kappa.
\end{equation*}
\STATE $w_k \leftarrow \frac{v_k}{\twonorm{v_k}}$.

\ENDFOR

\STATE {\bfseries return} $\tilde{w} \gets w_K$.
\end{algorithmic}
\end{algorithm}

\subsubsection{Hyper-parameter setting}\label{subsec:param-setting}
We elaborate on our hyper-parameter setting that is used in Algorithm~\ref{alg:nasty} and in our analysis; note that such concrete setting also applies to Algorithm~\ref{alg:main}. Let $g(t) = c_2 \del[1]{ 2 t \exp(-t) + \frac{c_3 \pi}{4} \exp\del[1]{-\frac{c_4 t}{4\pi}}  + 16 \exp(-t) }$, where the constants are specified in Appendix~\ref{sec:app:constants}. Observe that there exists an absolute constant $\bar{c} \geq 8\pi / c_4$ satisfying $g(\bar{c}) \leq 2^{-8}\pi$, since the continuous function $g(t) \to 0$ as $t \to +\infty$ and all the involved quantities in $g(t)$ are absolute constants. Given such constant $\bar{c}$, we set the constant $\kappa = \exp(-\bar{c})$, $r_1 = 1$ and $r_k = 2^{-k-6}$ for $k \geq 2$, $b_k = \bar{c} \cdot r_k$, $\tau_k = c_0 \kappa \cdot  \min\{b_k, 1/9\}$, $\delta_k = \frac{\delta}{(k+1)(k+2)}$, and choose $\xi_k = \min\big\{\frac{1}{2}, \frac{\kappa^2}{16} \del[1]{1 + 4\sqrt{C_2} {z_k}/{\tau_k} }^{-2}\big\}$ where $z_k = \sqrt{b_k^2 + r_k^2}$. It is easy to see that all $\xi_k$'s are lower bounded by a constant $c_6 \defeq \min\Big\{\frac{1}{2}, \frac{\kappa^2}{16} \del[2]{ 1 + \frac{4 }{c_0 \kappa \bar{c}} \sqrt{C_2 \bar{c}^2 + C_2} }^{-2} \Big\}$ and are upper bounded by $\frac{1}{2}$, thus they behave as $\frac{1}{2} - \Theta(1)$. Our theoretical guarantee holds for any noise rate $\eta \leq c_5 \epsilon$, where the constant $c_5 := \frac{c_8}{2\pi}\bar{c} c_1 c_6$.

We set the total number of phases $K = \log\del[1]{\frac{\pi}{32 c_1 \epsilon}}$. For any phase $k \geq 1$, we set $N_k = \frac{d}{b_k} \cdot \polylog{d, \frac{1}{b_k}, \frac{1}{\delta_k}}$ which is the number of instances requested by the learner.

%

\subsection{Analysis}

Throughout the section, we always presume that we are addressing the nasty noise model under Assumption~\ref{as:x}.

We decompose $A = \AC \cup \AD$, where $\AC$ is the set of clean instances in $A$ and $\AD$ consists of the dirty instances in $A$. Let $A'$ be the unlabeled clean instance set obtained by removing all labels in $\hat{A}'$. We introduce the instance set $\AE = A' \backslash \AC$, which was erased from $A'$ by the adversary.

The following lemma follows directly from the noise model and the Chernoff bound, which states that there are not too many dirty instances in $A$.

\begin{lemma}\label{lem:nasty-|AD|}
Consider the nasty noise model with noise rate $\eta \leq c_5\epsilon$. Then $\abs{\AD} \leq \frac{1}{2}c_8 \xi b N$ and $\abs{\AC} \geq (1- \frac{1}{2}c_8 \xi b)N$.
\end{lemma}


Next, we have an important consequence showing that when localizing the instances in the band $X_{u, b}$, the nasty noise rate stays as a small constant and there are sufficient clean instances in $A$ that are retained.

\begin{lemma}\label{lem:nasty-sample-set-size}
Let $\eta \leq c_5 \epsilon$. By calling $\nastyoraclex$, the following hold simultaneously with probability $1- \delta$:
\begin{enumerate}

\item \label{item:nasty-emp-noise} $\frac{\abs{\TD}}{\abs{T}} \leq \xi$;
\item \label{item:nasty-TC-size} $\abs{\TC} \geq \frac{1}{2}c_8 (1-\xi) b N$ and $\abs{\TE}  \leq \frac{1}{2} c_8 \xi b N$.
\end{enumerate}
\end{lemma}

By Part~\ref{item:nasty-TC-size} of the above lemma, we know that $\abs{\TC \cup \TE} \geq \Omega(bN)$. Hence results similar to Proposition~\ref{prop:var} immediately hold on the i.i.d. instance set $\TC \cup \TE$ provided that $N$ is large enough.

\begin{proposition}\label{prop:var-nasty}
Let $M = \frac{1}{\abs{\TC \cup \TE}} \sum_{x \in \TC \cup \TE} xx\trans$. If $N \geq \frac{d}{b}\cdot \polylog{d, \frac{1}{\delta}}$, then with probability $1-\delta$, $\lambdamax(M) \leq O(1)$.
\end{proposition}

The above proposition suffices to show that results similar to Theorem~\ref{thm:var} hold on $\TC \cup \TE$. Thus, if the learner were given $T \cup \TE$, then Proposition~\ref{prop:var-nasty} would imply the success of soft outlier removal under the nasty noise. Nevertheless, $\TE$ is in reality inaccessible to the learner; we will hence need a more careful analysis to establish the performance guarantee, which is the theme of the next theorem.

\begin{theorem}\label{thm:outlier-guarantee-nasty}
Let $\eta \leq c_5 \epsilon$ and $N \geq \frac{d}{b} \cdot \polylog{d, \frac{1}{\delta}}$. With probability $1- \delta$, Algorithm~\ref{alg:reweight} outputs a function $q: T \rightarrow [0, 1]$ in polynomial time with the following properties:
\begin{enumerate}
\item \label{item:outlier-nasty:q} for all $x \in T,\ q(x) \in [0, 1]$;
\item \label{item:outlier-nasty:avg-q} $\frac{1}{\abs{T}} \sum_{x \in T} q(x) \geq 1 - \xi$;


\item  \label{item:outlier-nasty:avg-var} $\sup_{w \in W} \frac{1}{\abs{T}} \sum_{x \in T} q(x) (w\cdot x)^2 \leq c \del{b^2 + r^2}$.
\end{enumerate}
\end{theorem}

Observe that the above theorem already guarantees an $\tilde{O}(d / b)$ sample complexity bound for the success of soft outlier removal. It remains to show that the output of Algorithm~\ref{alg:nasty} has small error rate with respect to $D$ and $w^*$. To this end, we need to characterize the performance of hinge loss minimization. Our approach is to link the reweighted hinge loss over $T$ to the hinge loss over $\TC \cup \TE$. This is because the latter is a good approximation to the expected hinge loss on clean samples in light of uniform concentration, which itself acts as a surrogate of a localized error rate (that is of our interest).

Let $\hatTC = \{ (x, \sign{w^* \cdot x}): x \in \TC\}$ be the (unrevealed) labeled set of $\TC$ (note that $\TC$ only contains clean instances, hence they are labeled correctly by the adversary); likewise we  denote by $\hatTE$ the (unrevealed) labeled set of $\TE$. Define $\ell_{\tau}(w; p \circ \hat{T}) = \frac{1}{\abs{T}} \sum_{x \in T} p(x) \cdot \max\big\{0,  1 - \frac{1}{\tau} y_x w \cdot x \big\}$ be the reweighted hinge loss over $T$ where $y_x$ denotes the label of $x$ that the adversary is committed to and $p(x)$ was calculated in Step~7 of Algorithm~\ref{alg:nasty}.

\begin{proposition}\label{prop:l(TC)=l(p)-nasty}
Let $\eta \leq  c_5\epsilon$. If $N \geq \frac{d}{b} \cdot \polylog{d, \frac{1}{\delta}}$, then with probability $1- \delta$, 
\begin{equation*}
\sup_{w \in W}\abs{\ell_{\tau}(w; \hatTC \cup \hatTE) - \ell_{\tau}(w; p \circ \hat{T})} \leq \kappa,
\end{equation*}
where $\kappa$ was defined in Section~\ref{subsec:param-setting}.
\end{proposition}


The above robust approximation of hinge loss combined with standard uniform concentration bounds (which require sample complexity $\tilde{O}(d)$) suffices to establish the following key lemma: the error rate within the band is a constant.

\begin{lemma}\label{lem:err_k(v_k)}
Let $\eta \leq  c_5\epsilon$. Consider phase $k$ of Algorithm~\ref{alg:nasty} with hyper-parameter settings in Section~\ref{subsec:param-setting}. If $w^* \in W_k$, then with probability $1- \frac{\delta}{(k+1)(k+2)}$,
\begin{equation*}
\err_{D_{w_{k-1}, b_k}}(v_k) \leq 6\kappa.
\end{equation*}
\end{lemma}

It is important to note a small  constant error rate within the band implies an $O(\epsilon)$ error rate of the final output $\tilde{w}$ over the distribution $D$~--~a well-known fact in margin-based active learning framework \cite{balcan2007margin,awasthi2017power}. Therefore, Lemma~\ref{lem:err_k(v_k)} has an immediate implication of the correctness of Theorem~\ref{thm:nasty-informal}; the full proof can be found in Appendix~\ref{sec:app:main-proof}.

\section{Conclusion}\label{sec:conclusion}

This paper provides an improved analysis on the sample complexity of a well-established algorithm for learning of homogeneous halfspaces under the malicious noise. It is shown that by leveraging a matrix Chernoff-type inequality with localization, the obtained sample complexity is optimal up to logarithmic factors. We also  extend our analysis to the stronger nasty noise model, and show the achievability of near-optimal noise tolerance and sample complexity by an efficient algorithm when the learner is permitted to communicate with the adversary for multiple rounds.

%

\clearpage

\bibliographystyle{alpha}
\bibliography{../../jshen_ref}

\newpage

\appendix

\section{Detailed Choices of Reserved Constants}\label{sec:app:constants}

The absolute constants $c_0$, $c_1$ and $c_2$ are specified in Lemma~\ref{lem:logconcave}, and $c_3$ and $c_4$ are specified in Lemma~\ref{lem:err outside band}. $c_5$ and $c_6$ are clarified in Section~\ref{subsec:param-setting}. The definition of $c_7$ and $c_8$ can be found in Lemma~\ref{lem:x-2norm} and Lemma~\ref{lem:P(x in band)} respectively. The absolute constant $C_1$ acts as an upper bound of all $b_k$'s, and by our choice in Section~\ref{subsec:param-setting}, $C_1 = \bar{c}/16$. The absolute constant $C_2$ is defined in Lemma~\ref{lem:E[wx^2]}. Other absolute constants, such as $C_3, C_4$ are not quite crucial to our analysis or algorithmic design. Therefore, we do not track their definitions. The subscript variants of $K$, e.g. $K_1$ and $K_2$, are also absolute constants but their values may change from appearance to appearance. We remark that the value of all these constants does not depend on the underlying distribution $D$ chosen by the adversary, but rather depends on the knowledge that $D$ is a member of the family of isotropic log-concave distributions.

\section{Omitted Proofs from Section~\ref{sec:mal}}\label{sec:app:proof-mal}

We will frequently use the well-known Chernoff bound in our analysis. For convenience, we record it below.

\begin{lemma}[Chernoff bound]\label{lem:chernoff}
Let $Z_1, Z_2, \dots, Z_n$ be $n$ independent random variables that take value in $\{0, 1\}$. Let $Z = \sum_{i=1}^{n} Z_i$. For each $Z_i$, suppose that $\Pr(Z_i =1) \leq \eta$.  Then for any $\alpha \in [0, 1]$
\begin{equation*}
\Pr\( Z \geq  (1+\alpha) \eta n\) \leq e^{-\frac{\alpha^2 \eta n}{3} }.
\end{equation*}
When $\Pr(Z_i =1) \geq \eta$, for any $\alpha \in [0, 1]$
\begin{equation*}
\Pr\( Z \leq  (1-\alpha) \eta n\) \leq e^{-\frac{\alpha^2 \eta n}{2} }.
\end{equation*}
\end{lemma}

\subsection{Proof of Lemma~\ref{lem:rho}}
\begin{proof}
We note that $(\rho^+ - \rho^-)^2 \leq 4 (\rho^+)^2$. In addition, this inequality is almost tight up to a constant factor since $\rho^-$ can be as small as $0$. To see this, observe that $u \in W$ and $x$ is such that $\abs{u \cdot x} \leq b$.

Thus, it remains to upper bound $\rho^+$. Due to localized sampling, for any $w \in W$ we have
\begin{equation}\label{eq:wx-abs}
\abs{w \cdot x} \leq \abs{(w - u) \cdot x} + \abs{u \cdot x} \leq \twonorm{w-u} \cdot \twonorm{x} + b \leq r \cdot c_7 \sqrt{d} \log\frac{1}{b\delta} + b,
\end{equation}
where the first step follows from the triangle inequality, the second step uses Cauchy-Schwarz inequality and the fact $x \sim D_{u, b}$, and the last step applies Lemma~\ref{lem:x-2norm}. The lemma follows by noting that $r = \Theta(b)$.
\end{proof}


\subsection{Proof of Lemma~\ref{lem:spectrum}}

\begin{proof}
For any unit vector $v$, observe that $w := r v + u$ is such that $\twonorm{w - u} \leq r$. Hence,
\begin{align*}
\EXP\sbr{(v \cdot x)^2} &= \frac{1}{r^2} \EXP\sbr{(r \cdot v \cdot x)^2}\\
& \leq \frac{2}{r^2} \EXP\sbr{((r \cdot v + u) \cdot x)^2} + \frac{2}{r^2}\EXP \sbr{ (u \cdot x)^2 }\\
&\leq \frac{2}{r^2} \cdot C_2(b^2 + r^2) + \frac{2}{r^2} \cdot b^2\\
&\leq \frac{4 C_2 (b^2 + r^2)}{r^2},
\end{align*}
where in the second step we use the basic inequality $a_1^2 \leq 2 (a_1 - a_2)^2 + 2 a_2^2$, and in the third step we apply Lemma~\ref{lem:E[wx^2]}. This proves the first desired inequality.

Next, by Lemma~\ref{lem:x-2norm} we have  with probability $1 - \delta$, $\twonorm{x} \leq c_7 \sqrt{d} \log\frac{1}{b \delta}$. Then for any unit vector $v$, we have
\begin{equation*}
(v \cdot x)^2 \leq \twonorm{v}^2 \cdot \twonorm{x}^2 \leq c_7^2 \cdot d  \log^2\frac{1}{b \delta},
\end{equation*}
which implies the second desired inequality.
\end{proof}

\subsection{Proof of Proposition~\ref{prop:var}}

\begin{proof}
In Lemma~\ref{lem:matrix-chernoff}, we set $\alpha = 1$, $M_i = x_i x_i\trans$ where $x_i$ is the $i$-th instance in the set $\TC$. Lemma~\ref{lem:spectrum} implies that $\mumax \leq \frac{4{C}_2 (b^2+r^2)}{r^2}  \abs{\TC} \leq K \cdot \abs{\TC}$ for some constant $K > 0$ since $r = \Theta(b)$, and with probability $1-\delta$, $\Lambda \leq K_1 \cdot d \log^2\frac{\abs{\TC}}{b\delta}$ by union bound. By conditioning on these events and putting all pieces together, Lemma~\ref{lem:matrix-chernoff} asserts that with probability $1 - d \cdot (\frac{e}{4})^{ \frac{K}{K_1} \cdot \frac{\abs{\TC}}{ d \log^2\frac{\abs{\TC}}{b\delta} } }$,
\begin{equation}\label{eq:tmp:max-eigen}
\lambdamax\del[2]{\sum_{x \in S} x x\trans } \leq 2 K \cdot \abs{\TC}.
\end{equation}
Equivalently, the above holds with probability $1-\delta$ as long as $\abs{\TC} \geq K_2 d \log^2\frac{\abs{\TC}}{b\delta} \cdot \log\frac{d}{\delta}$ for some constant $K_2 > 0$.
\end{proof}

\subsection{Proof of Lemma~\ref{lem:N}}

\begin{proof}
By Lemma~\ref{lem:P(x in band)}
\begin{equation*}
\Pr_{x \sim D}(x \in X) \geq c_8 b.
\end{equation*}
This implies that
\begin{align*}
&\ \Pr_{x \sim \oraclex}(x \in X_{u, b} \ \text{and}\  x \text{ is clean}) \\
=&\ \Pr_{x \sim \oraclex}(x \in X_{u, b} \mid  x \text{ is clean}) \cdot \Pr_{x \sim \oraclex}( x \text{ is clean}) \geq c_8 b (1-\eta).
\end{align*}
We want to ensure that by drawing $N$ instances from $\oraclex$, with probability at least $1- \delta$, $n$ out of them fall into the band $X_{u, b}$. We apply the second inequality of Lemma~\ref{lem:chernoff} by letting $Z_i=\ind{x_i \in X_{u, b} \ \text{and}\ x_i \text{ is clean}}$ and $\alpha = 1/2$, and obtain
\begin{equation*}
\Pr\del{ \abs{\TC} \leq \frac{c_8 b (1-\eta)}{2} N } \leq \exp\del{-\frac{c_8 b (1-\eta) N}{8}},
\end{equation*}
where the probability is taken over the event that we make a number of $N$ calls to $\oraclex$. Thus, when $N \geq \frac{8}{c_8 b (1- \eta)}\del{ n + \ln\frac{1}{\delta} }$, we are guaranteed that at least $n$ samples from $\oraclex$ fall into the band $X_{u, b}$ with probability $1 - \delta$. The lemma follows by observing $\eta < \frac{1}{2}$.
\end{proof}

\subsection{Proof of Lemma~\ref{lem:pruning}}

This is a simplified version of Lemma~30 of \citet{shen2020attribute}.

\begin{proof}
We calculate the noise rate within the band $X_k \defeq \{x: \abs{w_{k-1} \cdot x} \leq b_k\}$ by Lemma~\ref{lem:noise-rate-in-band}:
\begin{equation*}
\Pr_{x \sim \oraclex}( x\ \text{is\ dirty} \mid x \in X_{u, b}) \leq \frac{2\eta}{c_8 b} \leq \frac{2\eta}{c_8  \epsilon} \leq \frac{2c_5}{c_8} \leq \frac{1}{8},
\end{equation*}
where the second inequality applies the setting $b \geq \epsilon$, the third inequality is due to the condition $\eta \leq c_5 \epsilon$, and the last inequality is due to the condition that $c_5$ is assumed to be a sufficiently small constant. Now we apply the first inequality of Lemma~\ref{lem:chernoff} by specifying $Z_i = \ind{x_i\ \text{is\ dirty}}$, $\alpha = 1$ therein, which gives
\begin{equation*}
\Pr\del{ \abs{\TD} \geq \frac{1}{4} \abs{T} } \leq \exp\del{ - \frac{\abs{T}}{24}},
\end{equation*}
where the probability is taken over the draw of ${T}$. The lemma follows by setting the right-hand side to $\delta$ and noting that $\abs{\TC} = \abs{T} - \abs{\TD}$.
\end{proof}

\begin{lemma}\label{lem:noise-rate-in-band}
Assume $\eta < \frac{1}{2}$. We have
\begin{equation*}
\Pr_{x \sim \oraclex}\del{ x\ \text{is\ dirty} \mid x \in X_{u, b} } \leq \frac{2\eta}{c_8 b}
\end{equation*}
where $c_8$ was defined in Lemma~\ref{lem:P(x in band)}.
\end{lemma}
\begin{proof}
For an instance $x$, we use $\mathrm{tag}_x = 1$ to denote that $x$ is drawn from $D$, and use $\mathrm{tag}_x = -1$ to denote that $x$ is adversarially generated.

We first calculate the probability that an instance returned by $\oraclex$ falls into the band $X_{u, b}$ as follows:
\begin{align*}
&\ \Pr_{x \sim \oraclex} \del{x \in X_{u, b}} \\
=&\ \Pr_{x \sim \oraclex}\del{x \in X_{u, b} \ \text{and}\ \mathrm{tag}_x = 1} + \Pr_{x \sim \oraclex}\del{x \in X_{u, b} \ \text{and} \ \mathrm{tag}_x = -1}\\
\geq&\ \Pr_{x \sim \oraclex}\del{x \in X_{u, b} \ \text{and}\ \mathrm{tag}_x = 1}\\
=&\ \Pr_{x \sim \oraclex}\del{x \in X_{u, b} \mid \mathrm{tag}_x = 1} \cdot \Pr_{x \sim \oraclex}\del{\mathrm{tag}_x = 1}\\
=&\ \Pr_{x \sim D}\del{x \in X_{u, b}} \cdot \Pr_{x \sim \oraclex}\del{\mathrm{tag}_x = 1}\\
\stackrel{\zeta}{\geq}&\ c_8 b \cdot (1- \eta)\\
\geq&\ \frac{1}{2}c_8 b,
\end{align*}
where in the inequality $\zeta$ we applied Part~\ref{item:prob-band:refined-lower} of Lemma~\ref{lem:P(x in band)}. It is thus easy to see that
\begin{equation*}
\Pr_{x \sim \oraclex}\del{\textrm{tag}_x = -1 \mid x \in X_{u, b}} \leq \frac{ \Pr_{x \sim \oraclex}\del{ \textrm{tag}_x = -1} }{ \Pr_{x \sim \oraclex} \del{x \in X_{u, b}} } \leq \frac{2\eta}{c_8 b},
\end{equation*}
which is the desired result.
\end{proof}


\subsection{Rademacher analysis leads to suboptimal sample complexity for quadratic functions}

To see why a general Rademacher analysis may not suffice, we can, for example, think of the quadratic function $(w \cdot x)^2$ as a composition of the functions $\phi(f) = f^2$ and $f_w(x) = w\cdot x$. Recall that we showed with high probability that $\abs{w\cdot x} \leq O(b \sqrt{d})$ (omitting logarithmic factors for convenience). Now, the gradient of $\phi(\cdot)$ is $2 w \cdot x$ which is upper bounded by ${O}(b \sqrt{d})$ and the function value of $\phi(\cdot)$ is upper bounded by ${O}(b^2 d)$. For the Rademacher complexity $\calR_\calF$ of the class of linear functions $\calF := \{ f_w(x) = w\cdot x: w \in W\}$ on $\TC = \{x_1, \dots, x_n\}$, let $V = \{v \in \Rd: \twonorm{v} \leq 1$ and note that for any $w \in W$, $w = u + r v$. We have by definition
\begin{align*}
\calR_\calF &= \frac{1}{n} \EXP \sup_{w \in W} \sum_{i=1}^n \sigma_i (w \cdot x_i)\\
&= \frac{1}{n} \EXP \sup_{w \in W} {w} \cdot { \sum_{i=1}^n \sigma_i  x_i }\\
&\leq \frac{r}{n} \EXP \sup_{v \in V} v \cdot { \sum_{i=1}^n \sigma_i  x_i } + \frac{1}{n} \EXP u \cdot { \sum_{i=1}^n \sigma_i  x_i }\\
&\leq \frac{r}{n} \EXP \twonorm{  \sum_{i=1}^n \sigma_i  x_i }\\
&\leq \frac{r}{n} \cdot \sqrt{n} \max_{1 \leq i \leq n} \twonorm{x_i},
\end{align*}
where the expectation is taken over the i.i.d. Rademacher variables $\sigma_1, \dots, \sigma_n$. By Lemma~\ref{lem:x-2norm}, $\calR_\calF \leq \frac{r}{\sqrt{n}} \sqrt{d}$ with high probability. By the contraction lemma, the Rademacher complexity of the class of quadratic functions is $O(\frac{br d}{\sqrt{n}})$, and thus uniform concentration through Rademacher analysis requires $O(d^2)$ samples.

Similarly, a straightforward application of local Rademacher analysis~\cite{bartlett2005local} may not suffice as well. However, our discussion here does not rule out the possibility that a more sophisticated exploration of these techniques would lead to the desired sample complexity bound; we leave it as an open problem.

\section{Omitted Proofs from Section~\ref{sec:nasty}}\label{sec:app:nasty}

We present a full proof of the results in Section~\ref{sec:nasty}. Observe that the malicious noise is a special case of the nasty noise; hence this section can also be thought of as providing a complete proof for the results in Section~\ref{sec:mal}.

To improve the transparency, we collect useful notations in Table~\ref{tb:notation}.

\begin{table}[h]
\centering
\caption{Summary of useful notations associated with the working set ${T}$ at each phase $k$ for learning with nasty noise.}
\vspace{1em}
\begin{tabular}{ll}
\toprule
$\hat{A}'$ & labeled clean instance set obtained by drawing $N$ instances from $D$ and labeling them by $w^*$\\
$A'$ & (clean) instance set obtained by hiding all the labels in $\hat{A}'$\\
$\hat{A}$ & labeled corrupted instance set obtained by replacing $\eta N$ samples in $\hat{A}'$\\
$A$ & (corrupted) instance set obtained by hiding all the labels in $\hat{A}$\\
$\AC$ & set of clean instances in $A$\\
$\AD$ & set of dirty instances in $A$, i.e. $A \backslash \AC$\\
$\AE$ & set of clean instances erased from ${A}'$ by the adversary\\
${T}$ & set of instances in $A$ that satisfy $\abs{w_{k-1} \cdot x} \leq b_k$\\
$\TC$ & set of clean instances in ${T}$\\
$\TD$ & set of dirty instances in ${T}$, i.e. ${T} \backslash \TC$\\
$\hatTC$ & unrevealed labeled set of $\TC$\\
$\hatTE$ & unrevealed labeled set of $\TE$\\
\bottomrule
\end{tabular}
\label{tb:notation}
\end{table}

\subsection{Proof of Lemma~\ref{lem:nasty-|AD|}}

\begin{proof}
Since $\eta \leq c_5 \epsilon$ and $b \geq \epsilon$, we have $\eta \leq c_5 b \leq \frac{1}{2}c_8 \xi b$ where the second inequality follows from the fact that $c_5$ is a small constant and $\xi \geq \Omega(1)$. Thus $\abs{\AD} = \eta N \leq \frac{1}{2}c_8 \xi b N$ and $\abs{\AC} = N - \abs{\AD} \geq (1- \frac{1}{2}c_8 \xi b)N$.
\end{proof}

\subsection{Proof of Lemma~\ref{lem:nasty-sample-set-size}}

\begin{proof}
We first show that the following two events hold simultaneously with probability $1-\frac{\delta_k}{24}$:
\begin{align*}
&E_1: \abs{\AC} \geq \del[2]{1- \frac{1}{2} c_8 \xi b}N \ \text{and}\ \abs{\AD} \leq \frac{1}{2} c_8 \xi b N,\\
&E_2: \abs{\TC} \geq \frac{1}{2}c_8 (1-\xi) b N \ \text{and}\ \abs{\TE} \leq \frac{1}{2} c_8 \xi b N.
\end{align*}
Observe that $E_1$ holds with certainty due to Lemma~\ref{lem:nasty-|AD|}.

To see why $E_2$ holds with high probability, we recall that Part~\ref{item:prob-band:refined-lower} of Lemma~\ref{lem:P(x in band)} shows that $\Pr_{x \sim D}\del{x \in X_{u, b} } \geq c_8 b$. For each $x_i \in \AC \cup \AE$, define $Z_i = \ind{ x_i \in X_{u, b}}$. Since $\AC \cup \AE$ are i.i.d. draws from $D$, by applying the second part of Lemma~\ref{lem:chernoff} with $\alpha = 1/2$, we have
\begin{equation*}
\Pr\del[3]{ \sum_{i=1}^{N} Z_i \leq \frac{1}{2} c_8 b N } \leq \exp\del{- \frac{c_8 b N}{8}}.
\end{equation*}
This shows that
\begin{equation*}
\abs{\TC} + \abs{\TE} \geq \frac{1}{2}c_8 b N
\end{equation*}
with probability $1- \delta$ provided that $N \geq \frac{8}{c_8 b}\ln\frac{1}{\delta}$. On the other side, we have $\abs{\TE} \leq \abs{\AE} = \abs{\AD} \leq \frac{1}{2} c_8 \xi b N$. Thus it follows that $\abs{\TC} \geq \frac{1}{2}c_8 (1-\xi) b N$.

For Part~\ref{item:nasty-emp-noise}, we have
\begin{equation}
\frac{\abs{\TC}}{\abs{\TD}} \geq \frac{1-\xi}{\xi},
\end{equation}
where the inequality follows from  $E_2$ and the fact $\abs{\TD} = \abs{\TE}$. Therefore,
\begin{equation}
\frac{\abs{\TD}}{\abs{{T}}} = \frac{1}{1 + \abs{\TC} / \abs{\TD}} \leq \xi.
\end{equation}

Part~\ref{item:nasty-TC-size} of the lemma simply follows $E_2$.
\end{proof}

\subsection{Proof of Proposition~\ref{prop:var-nasty}}
\begin{proof}
Since $N \geq \frac{d}{b} \cdot \polylog{(d, \frac{1}{\delta}}$, we have by Part~\ref{item:nasty-TC-size} that $\abs{\TC \cup \TE} \geq \abs{\TC} \geq d \cdot \polylog{d, \frac{1}{\delta}}$. Therefore, we can directly apply Proposition~\ref{prop:var} by thinking of $\TC$ therein as $\TC \cup \TE$ in the current proposition.
\end{proof}

\subsection{Proof of Theorem~\ref{thm:outlier-guarantee-nasty}}

\begin{proof}
We first show the existence of a feasible function $q(x)$ to Algorithm~\ref{alg:reweight}. Consider the specific function $q: T \rightarrow [0, 1]$ as follows: $q(x) = 1$ for all $x \in \TC$ and $q(x) = 0$ otherwise. We have
\begin{equation*}
\frac{1}{\abs{T}} \sum_{x \in T} q(x) = \frac{\abs{\TC}}{\abs{T}} = 1 - \frac{\abs{\TD}}{\abs{T}} \geq 1 - \xi,
\end{equation*}
in view of Part~\ref{item:nasty-emp-noise} of Lemma~\ref{lem:nasty-sample-set-size}.

To show Part~\ref{item:outlier-nasty:avg-var}, we note that $\TC \cup \TE$ are i.i.d. draws from $D_{u, b}$ and Lemma~\ref{lem:nasty-sample-set-size} shows that $\abs{\TC \cup \TE} \geq \Omega(b N)$. Therefore, as far as $N \geq \frac{d}{b} \cdot \polylog{d}$, Theorem~\ref{thm:var} implies that
\begin{equation*}
\frac{1}{\abs{\TC} + \abs{ \TE}} \sum_{x \in \TC \cup \TE} (w \cdot x)^2 \leq \frac{c}{2} (b^2 + r^2).
\end{equation*}
Since $(w\cdot x)^2$ is always non-negative, we have
\begin{equation*}
\frac{1}{\abs{\TC}} \sum_{x \in \TC} (w\cdot x)^2 \leq \frac{\abs{\TC} + \abs{ \TE}}{\abs{\TC}} \cdot \frac{1}{\abs{\TC} + \abs{ \TE}} \sum_{x \in \TC \cup \TE} (w\cdot x)^2 \leq \frac{\abs{\TC} + \abs{ \TE}}{\abs{\TC}} \cdot \frac{c}{2} (b^2 + r_k^2).
\end{equation*}
Part~\ref{item:nasty-TC-size} of Lemma~\ref{lem:nasty-sample-set-size} shows that $\abs{\TE} / \abs{\TC} \leq \frac{\xi}{1-\xi} \leq 1$ since $\xi \leq \frac{1}{2}$. Plugging this upper bound into the above inequality, we obtain
\begin{equation*}
\frac{1}{\abs{\TC}} \sum_{x \in \TC} (w \cdot x)^2 \leq c (b^2 + r^2).
\end{equation*}
In a nutshell, our construction of $q(x)$ ensures the feasibility to all constraints in Algorithm~\ref{alg:reweight}. By ellipsoid method we are able to find a feasible solution in polynomial time.
\end{proof}

\subsection{Proof of Proposition~\ref{prop:l(TC)=l(p)-nasty}}
Let $z = \sqrt{b^2 + r^2}$. We will in fact prove a stronger result, i.e.,
\begin{align}
\ell_{\tau}(w; \hatTC \cup \hatTE) &\leq \ell_{\tau}(w; p \circ \hat{T}) + 2 \xi \del{2 + \sqrt{2K_2} \cdot \frac{z}{\tau}}  + \sqrt{2K_2 \xi} \cdot \frac{z}{\tau},\\
\ell_{\tau}(w; p \circ \hat{T}) &\leq \ell_{\tau}(w; \hatTC \cup \hatTE) + 2 \xi + \sqrt{4K_2 \xi} \cdot \frac{z}{\tau}.
\end{align}
The claim in the proposition immediately follows since $z/\tau = \Theta(1)$ and $\xi$ can be chosen as an arbitrarily small constant.

Let $\{q(x)\}_{x \in T}$ be the output of Algorithm~\ref{alg:reweight} under the nasty noise model. We extend the domain of $q(x)$ from $T$ to $T \cup \TE$ as follows: for any $x \in T$, the value $q(x)$ remains unchanged; for any $x \in \TE$, we set $q(x) = 0$. With this in mind, we can, for the purpose of analysis, think of the probability mass function $\{p(x)\}_{x \in T}$ obtained in Algorithm~\ref{alg:main} as over $T \cup \TE$, with the value $p(x)$ stays unchanged for $x \in T$ and $p(x) = 0$ for all $x \in \TE$. 

Now with the {\em extended} probability mass function $\{p(x)\}_{x \in T \cup \TE}$, we can prove the proposition.

\begin{proof}
Let $\hatTC$ and $\hatTE$ be the labeled set of $\TC$ and $\TE$ that is correctly annotated by $w^*$ respectively. For any $x$ in the instance space, let $y_x$ be the label that the adversary is committed to. Recall that the empirical distribution $\{ p(x) \}_{x \in T \cup \TE}$ was defined as follows: $p(x) = \frac{q(x)}{\sum_{x \in T} q(x)}$ for $x \in T$ and $p(x) = 0$ for $x \in \TE$. The reweighted hinge loss on $T \cup \TE$ using $p(x)$ is given by
\begin{equation}
\ell_{\tau}(w; p \circ \hat{T}) = \frac{1}{\abs{T \cup \TE}} \sum_{x \in T\cup \TE} p(x) \cdot \max\Big\{ 0, 1 - \frac{1}{\tau} y_x w \cdot x \Big\}.
\end{equation}

The choice of $N$ guarantees that Proposition~\ref{prop:var-nasty}, Lemma~\ref{lem:nasty-sample-set-size}, and Theorem~\ref{thm:outlier-guarantee-nasty} hold simultaneously with probability $1- \delta$. We thus have for all $w \in W$
\begin{align}
\frac{1}{\abs{\TC \cup \TE}}\sum_{x \in \TC \cup \TE}(w \cdot x)^2 \leq K_1 z^2,\label{eq:tmp:wx}\\
\frac{\abs{\TD}}{\abs{T}} \leq \xi,\label{eq:tmp:xi}\\
\frac{1}{\abs{T}} \sum_{x \in T} q(x)(w\cdot x)^2 \leq K_2 z^2.\label{eq:tmp:qx-wx}
\end{align}
We now expand $T$ to $T \cup \TE$ for the last two inequalities. Indeed, from \eqref{eq:tmp:xi}, it is easy to show that
\begin{equation}\label{eq:tmp:xi-2}
\frac{\abs{\TD}}{\abs{T \cup \TE}} \leq \frac{\abs{\TD}}{\abs{T}} \leq  \xi.
\end{equation}
Next, since we defined $q(x) = 0$ for all $x \in \TE$, \eqref{eq:tmp:qx-wx} implies that
\begin{equation}\label{eq:tmp:qx-wx-2}
\frac{1}{\abs{T \cup \TE}} \sum_{x \in T \cup \TE} q(x) (w \cdot x)^2 = \frac{1}{\abs{T \cup \TE}} \sum_{x \in T} q(x) (w \cdot x)^2 \leq \frac{1}{\abs{T}} \sum_{x \in T} q(x)(w\cdot x)^2 \leq K_2 z^2.
\end{equation}
The remaining steps are exactly same as Proposition~33 of \citet{shen2020attribute} since all the analyses therein rely only on the conditions \eqref{eq:tmp:wx}, \eqref{eq:tmp:xi-2} and \eqref{eq:tmp:qx-wx-2}. For completeness, we present the full proof here.

It follows from Eq.~\eqref{eq:tmp:xi-2} and $\xi \leq 1/2$ that
\begin{equation}\label{eq:tmp:|W|/|T_C|}
\frac{\abs{T \cup \TE}}{\abs{\TC \cup \TE}} \leq \frac{\abs{T \cup \TE}}{\abs{\TC}}  = \frac{\abs{T\cup \TE}}{\abs{T\cup \TE} - \abs{\TD}}  = \frac{1}{1 - \abs{\TD} / \abs{T\cup \TE}} \leq \frac{1}{1- \xi} \leq 2.
\end{equation}
In the following, we condition on the event that all these inequalities are satisfied.

\vspace{0.1in}
\noindent{\bfseries Step 1.}
First we upper bound $\ell_{\tau}(w; \hatTC \cup \hatTE)$ by $\ell_{\tau}(w; p \circ \hat{T})$.
\begin{align}
\abs{\TC \cup \TE} \cdot \ell_{\tau}(w; \hatTC \cup \hatTE) &=  \sum_{x \in \TC \cup \TE} \ell(w; x, y_x) \notag\\
&= \sum_{x \in T \cup \TE} \sbr{ q(x) \ell(w; x, y_x) + \big(\ind{x \in \TC \cup \TE} - q(x)\big) \ell(w; x, y_x) } \notag\\
&\stackrel{\zeta_1}{\leq}   \sum_{x \in T \cup \TE} q(x) \ell(w; x, y_x) + \sum_{x \in \TC \cup \TE} (1 - q(x)) \ell(w; x, y_x) \notag\\
&\stackrel{\zeta_2}{\leq}  \sum_{x \in T \cup \TE} q(x) \ell(w; x, y_x) + \sum_{x \in \TC \cup \TE} (1 - q(x)) \del{1 + \frac{\abs{w\cdot x}}{\tau}} \notag\\
&\stackrel{\zeta_3}{\leq} \sum_{x \in T \cup \TE} q(x) \ell(w; x, y_x) + \xi \abs{T \cup \TE} + \frac{1}{\tau} \sum_{x \in \TC \cup \TE} (1-q(x)) \abs{w\cdot x} \notag\\
&\stackrel{\zeta_4}{\leq} \sum_{x \in T \cup \TE} q(x) \ell(w; x, y_x) \notag\\
&\quad+ \xi \abs{T \cup \TE} + \frac{1}{\tau} \sqrt{\sum_{x \in \TC \cup \TE} (1-q(x))^2} \cdot \sqrt{\sum_{x \in \TC \cup \TE} (w\cdot x)^2} \notag\\
&\stackrel{\zeta_5}{\leq} \sum_{x \in T \cup \TE} q(x) \ell(w; x, y_x) + \xi \abs{T \cup \TE} + \frac{1}{\tau} \sqrt{\xi \abs{T \cup \TE}} \cdot  \sqrt{K_1 \abs{\TC \cup \TE}} \cdot {z}, \label{eq:tmp:l(T_C)-1}
\end{align}
where $\zeta_1$ follows from the simple fact that 
\begin{align*}
\sum_{x \in T \cup \TE} \del[1]{\ind{x \in \TC \cup \TE} - q(x)} \ell(w; x, y_x) &= \sum_{x \in \TC\cup \TE} (1-q(x)) \ell(w; x, y_x) + \sum_{x \in \TD} (-q(x)) \ell(w; x, y_x) \\
&\leq \sum_{x \in \TC\cup \TE} (1-q(x)) \ell(w; x, y_x),
\end{align*}
$\zeta_2$ explores the fact that the hinge loss is always upper bounded by $1 + \frac{\abs{w\cdot x}}{\tau}$ and that $1-q(x) \geq 0$, $\zeta_3$ follows from Part~\ref{item:outlier-nasty:avg-q} of Theorem~\ref{thm:outlier-guarantee-nasty}, $\zeta_4$ applies Cauchy-Schwarz inequality, and $\zeta_5$ uses Eq.~\eqref{eq:tmp:wx}.

In view of Eq.~\eqref{eq:tmp:|W|/|T_C|}, we have $\frac{\abs{T \cup \TE}}{\abs{\TC \cup \TE}} \leq 2$. Continuing Eq.~\eqref{eq:tmp:l(T_C)-1}, we obtain
\begin{align}
\ell_{\tau}(w; \hatTC \cup \hatTE) &\leq \frac{1}{\abs{\TC \cup \TE}}\sum_{x \in T \cup \TE} q(x) \ell(w; x, y_x) + 2 \xi + \sqrt{2K_1 \xi} \cdot \frac{z}{\tau} \notag\\
&= \frac{\sum_{x \in T \cup \TE} q(x)}{\abs{\TC \cup \TE}} \sum_{x \in T \cup \TE} p(x) \ell(w; x, y_x) + 2 \xi + \sqrt{2 K_1 \xi} \cdot \frac{z}{\tau} \notag\\
&= \ell_{\tau}(w; p \circ \hat{T}) + \del{\frac{\sum_{x \in T \cup \TE} q(x)}{\abs{\TC \cup \TE}} - 1} \sum_{x \in T \cup \TE} p(x) \ell(w; x, y_x) + 2 \xi + \sqrt{2 K_1 \xi} \cdot \frac{z}{\tau} \notag\\
&\leq \ell_{\tau}(w; p \circ \hat{T}) + \del{\frac{\abs{T \cup \TE}}{\abs{\TC \cup \TE}} - 1} \sum_{x \in T \cup \TE} p(x) \ell(w; x, y_x) + 2 \xi + \sqrt{2 K_1 \xi} \cdot \frac{z}{\tau} \notag\\
&\leq \ell_{\tau}(w; p \circ \hat{T}) + 2\xi \sum_{x \in T \cup \TE} p(x) \ell(w; x, y_x) + 2 \xi + \sqrt{2 K_1 \xi} \cdot \frac{z}{\tau},\label{eq:tmp:l(T_C)-2}
\end{align}
where in the last inequality we use the fact that $\abs{\TE} = \abs{\TD}$ and $T \cap \TE = \emptyset$, and thus
\begin{equation*}
\frac{\abs{T \cup \TE}}{\abs{\TC \cup \TE}} - 1 = \frac{\abs{T} + \abs{\TD}}{\abs{T}} - 1 = \frac{\abs{\TD}}{\abs{T}} \leq \xi.
\end{equation*}

On the other hand, we have the following result which will be proved later on.
\begin{claim}\label{claim:err-aux}
$\sum_{x \in T \cup \TE} p(x) \ell(w; x, y_x) \leq 1 + \sqrt{2K_2} \cdot \frac{z}{\tau}.$
\end{claim}

Therefore, continuing Eq.~\eqref{eq:tmp:l(T_C)-2} we have
\begin{equation*}
\ell_{\tau}(w; \hatTC \cup \hatTE) \leq \ell_{\tau}(w; p \circ \hat{T}) + 2 \xi \del{2 + \sqrt{2K_2} \cdot \frac{z}{\tau}}  + \sqrt{2K_2 \xi} \cdot \frac{z}{\tau}.
\end{equation*}
which proves the first inequality of the proposition.

\vspace{0.1in}
\noindent{\bfseries Step 2.}
We move on to prove the second inequality of the theorem, i.e. using $\ell_{\tau}(w; \hatTC \cup \hatTE)$ to upper bound $\ell_{\tau}(w; p \circ \hat{T})$. Let us denote by $\pD = \sum_{x \in \TD} p(x)$ the probability mass on dirty instances. Then
\begin{equation}\label{eq:tmp:p_D}
\pD = \frac{\sum_{x \in \TD} q(x)}{\sum_{x \in T} q(x)} \leq \frac{\abs{\TD}}{(1-\xi)\abs{T}}\leq \frac{\xi}{1-\xi} \leq 2 \xi,
\end{equation}
where the first inequality follows from $q(x) \leq 1$ and Part~\ref{item:outlier-nasty:avg-q} of Theorem~\ref{thm:outlier-guarantee-nasty}, the second inequality follows from \eqref{eq:tmp:xi}, and the last inequality is by our choice $\xi \leq 1/2$.

Note that by Part~\ref{item:outlier-nasty:avg-q} of Theorem~\ref{thm:outlier-guarantee-nasty} and the choice $\xi \leq 1/2$, we have
\begin{equation*}
\sum_{x \in T} q(x) \geq (1-\xi) \abs{T} \geq \abs{T}/2
\end{equation*}
Hence
\begin{align}\label{eq:tmp:sum p(x)(wx)^2}
\sum_{x \in T  } p(x) (w\cdot x)^2 &= \frac{1}{\sum_{x \in T } q(x)} \sum_{x \in T } q(x) (w\cdot x)^2 \notag\\
&\leq \frac{2}{\abs{T}}  \sum_{x \in T } q(x) (w\cdot x)^2 \notag\\
&\leq 2 \cdot K_2 z^2
\end{align}
where the last inequality holds because of \eqref{eq:tmp:qx-wx}. Thus,
\begin{align*}
\sum_{x \in \TD} p(x) \ell(w; x, y_x) &\leq \sum_{x \in \TD} p(x) \del{1 + \frac{\abs{w\cdot x}}{\tau}}\\
&= \pD + \frac{1}{\tau}\sum_{x \in \TD} p(x) \abs{w\cdot x}\\
&= \pD +\frac{1}{\tau} \sum_{x \in T} \del {\ind{x \in \TD} \sqrt{p(x)} } \cdot \del{ \sqrt{p(x)}\abs{w\cdot x} }\\
&\leq \pD +  \frac{1}{\tau}  \sqrt{\sum_{x \in T} \ind{x \in \TD} {p(x)}} \cdot \sqrt{ \sum_{x \in T} p(x) (w\cdot x)^2 }\\
&\stackrel{\eqref{eq:tmp:sum p(x)(wx)^2}}{\leq} \pD + \sqrt{\pD} \cdot \sqrt{2K_2} \cdot  \frac{z}{\tau}.
\end{align*}
With the result on hand, we bound $\ell_{\tau}(w; p \circ \hat{T})$ as follows:
\begin{align*}
\ell_{\tau}(w; p \circ \hat{T}) &= \sum_{x \in \TC \cup \TE} p(x) \ell(w; x, y_x) + \sum_{x \in \TD} p(x) \ell(w; x, y_x)\\
&\leq \sum_{x \in \TC \cup \TE} \ell(w; x, y_x) +  \sum_{x \in \TD} p(x) \ell(w; x, y_x)\\
&= \ell_{\tau}(w; \hatTC \cup \hatTE) +  \sum_{x \in \TD} p(x) \ell(w; x, y_x)\\
&\leq \ell_{\tau}(w; \hatTC\cup \hatTE) + \pD + \sqrt{\pD} \cdot \sqrt{2K_2} \cdot  \frac{z}{\tau}\\
&\stackrel{\eqref{eq:tmp:p_D}}{\leq} \ell_{\tau}(w; \hatTC \cup \hatTE) + 2 \xi + \sqrt{4K_2 \xi } \cdot \frac{z}{\tau},
\end{align*}
which proves the second inequality of the proposition.

This completes the proof.
\end{proof}

\begin{proof}[Proof of Claim~\ref{claim:err-aux}]
Since $\ell(w; x, y_x) \leq 1 + \frac{\abs{w\cdot x}}{\tau}$, it follows that
\begin{align*}
\sum_{x \in T \cup \TE} p(x) \ell(w; x, y_x) &\leq \sum_{x \in T \cup \TE} p(x) \del{ 1 + \frac{\abs{w\cdot x}}{\tau}}\\
&= 1 + \frac{1}{\tau} \sum_{x \in T \cup \TE} p(x) \abs{w\cdot x}\\
&\leq 1 + \frac{1}{\tau} \sqrt{ \sum_{x \in T \cup \TE} p(x) (w\cdot x)^2}\\
&\stackrel{\eqref{eq:tmp:sum p(x)(wx)^2}}{\leq} 1 + \sqrt{2K_2} \cdot \frac{z}{\tau},
\end{align*}
which completes the proof of Claim~\ref{claim:err-aux}.
\end{proof}

\subsection{Proof of Lemma~\ref{lem:err_k(v_k)}}

For any phase $k$, let $L_{\tau_k}(w) = \EXP_{x \sim D_{w_{k-1}, b_{k}}} \sbr{ \ell_{\tau_k}(w; x, \sign{w^* \cdot x}) }$.

\begin{proof}

Proposition~35 of \citet{shen2020attribute} showed that if $\abs{\TC \cup \TE} \geq d \cdot \polylog{d, \frac{1}{b}, \frac{1}{\delta}}$, then by Rademacher complexity of the hinge loss we have that with probability $1- \frac{\delta}{2}$
\begin{equation}
\sup_{w \in W}\abs{ \ell_{\tau}(w; \hatTC \cup \hatTE) - \EXP_{x \sim D_{u, b}}[ \ell_{\tau}(w; x, \sign{w^* \cdot x})] } \leq \kappa.
\end{equation}
Combining the above with Proposition~\ref{prop:l(TC)=l(p)-nasty} gives that with probability $1-\delta$,
\begin{equation*}
\sup_{w \in W} \abs{ \ell_{\tau}(w; p \circ \hat{T}) - \EXP_{x \sim D_{u, b}}[ \ell_{\tau}(w; x, \sign{w^* \cdot x})] } \leq 2 \kappa.
\end{equation*}
Namely, in any phase $k \leq K$, if $\abs{\TC \cup \TE} \geq d \cdot \polylog{d, \frac{1}{b_k}, \frac{1}{\delta_k}}$, then with probability $1-\delta_k$,
\begin{equation}
\sup_{w \in W_k} \abs{\ell_{\tau_k}(w; p) - L_{\tau_k}(w) } \leq 2 \kappa.
\end{equation}

On the other hand, since the (rescaled) hinge loss is always an upper bound of the error rate, we have
\begin{align*}
\err_{D_{w_{k-1}, b_{k}}}(v_k) \leq L_k(v_k) \stackrel{\zeta_1}{\leq} \ell_{\tau_k}(v_k; p) + 2 \kappa \stackrel{\zeta_2}{\leq} \min_{w \in W_k}\ell_{\tau_k}(w; p) + 3 \kappa \leq \ell_{\tau_k}(w^*; p) + 3\kappa &\stackrel{\zeta_3}{\leq} L_k(w^*) + 5\kappa\\
& \stackrel{\zeta_4}{\leq} 6\kappa \leq 8\kappa,
\end{align*}
where we use the fact that $v_k \in W_k$ in $\zeta_1$, use the optimality condition of $v_k$ in $\zeta_2$, use $w^* \in W_k$ in $\zeta_3$, and use Lemma~\ref{lem:L(w*)} in $\zeta_4$.
\end{proof}

\begin{lemma}[Lemma~3.7 in \citet{awasthi2017power}]\label{lem:L(w*)}
Suppose Assumption~\ref{as:x} is satisfied. Then
\begin{equation*}
L_{\tau_k}(w^*)  \leq \frac{\tau_k}{c_0 \min\{b_k, 1/9\}}.
\end{equation*}
In particular, by our choice of $\tau_k$, it holds that
\begin{equation*}
L_{\tau_k}(w^*) \leq \kappa.
\end{equation*}
\end{lemma}

\begin{lemma}\label{lem:feasible-to-angle}
For any $1 \leq k \leq K$, if $w^* \in W_k$, then with probability $1-\delta_k$, $\theta(v_k, w^*) \leq 2^{-k-8} \pi$.
\end{lemma}
\begin{proof}

For $k = 1$, by Lemma~\ref{lem:err_k(v_k)} with the facts that we actually sample from $D$ and $w^* \in  \Rd =: W_1$, we immediately have
\begin{equation*}
\Pr_{x \sim D}\(\sign{v_1 \cdot x} \neq \sign{w^* \cdot x}\) \leq 8\kappa.
\end{equation*}
Hence Part~\ref{item:ilc:err=theta} of Lemma~\ref{lem:logconcave} indicates that
\begin{equation}\label{eq:tmp:theta_1}
\theta(v_1, w^*) \leq 8c_2 \kappa = 16 c_2 \kappa \cdot 2^{-1}.
\end{equation}

Now we consider $2 \leq k \leq  K$. Denote $X_k = \{x: \abs{w_{k-1}\cdot x} \leq b_k\}$, and $\bar{X}_k = \{x: \abs{w_{k-1}\cdot x} > b_k\}$. We will show that the error of $v_k$ on both $X_k$ and $\bar{X}_k$ is small, hence $v_k$ is a good approximation to $w^*$.

First, we consider the error on $X_k$, which is given by
\begin{align}
 &\ \Pr_{x \sim D}\( \sign{v_k \cdot x} \neq \sign{w^* \cdot x}, x \in X_k \) \notag\\
=&\ \Pr_{x \sim D}\( \sign{v_k \cdot x} \neq \sign{w^* \cdot x} \mid x \in X_k \) \cdot \Pr_{x \sim D}(x \in X_k) \notag\\
=&\ \err_{D_{w_{k-1}, b_{k}}}(v_k) \cdot \Pr_{x \sim D}(x \in X_k) \notag\\
\leq&\ 8\kappa \cdot 2b_k = 16 \kappa b_k,\label{eq:tmp:err in band}
\end{align}
where the inequality is due to Lemma~\ref{lem:err_k(v_k)} and Lemma~\ref{lem:logconcave}. Note that the inequality holds with probability $1 - \delta_k$ in view of Lemma~\ref{lem:err_k(v_k)}.

Next we derive the error on $\bar{X}_k$. Note that Lemma~10 of \citet{zhang2018efficient} states for any unit vector $u$, and any general vector $v$, $\theta(v, u) \leq \pi \twonorm{v - u}$. Hence,
\begin{align*}
\theta(v_k, w^*)  \leq \pi \twonorm{v_k - w^*} \leq \pi (\twonorm{v_k - w_{k-1}} + \twonorm{w^* - w_{k-1}}) \leq 2 \pi r_k,
\end{align*}
where we use the condition that both $v_k$ and $w^*$ are in $W_k$.

Recall that we set $r_k = 2^{-k-6} < 1/4$ in our algorithm and choose $b_k = \bar{c} \cdot r_k$ where $\bar{c} \geq 8\pi / c_4$, which allows us to apply Lemma~\ref{lem:err outside band} and obtain
\begin{align*}
\Pr_{x \sim D}\( \sign{v_k \cdot x} \neq \sign{w^* \cdot x}, x \notin X_k \) &\leq c_3 \cdot 2\pi r_k \cdot \exp\del{- \frac{c_4 \bar{c} \cdot r_k}{2 \cdot 2\pi r_k}} \\
&= 2^{-k} \cdot  \frac{c_3 \pi}{4} \exp\del{- \frac{c_4 \bar{c}}{4 \pi}}.
\end{align*}
This in allusion to \eqref{eq:tmp:err in band} gives
\begin{equation*}
\err_D(v_k) \leq 16\kappa \cdot \bar{c} \cdot r_k + 2^{-k} \cdot  \frac{c_3 \pi}{4} \exp\del{- \frac{c_4 \bar{c}}{4 \pi}} = \del{ 2\kappa \bar{c} + \frac{c_3 \pi}{4} \exp\del{- \frac{c_4 \bar{c}}{4 \pi}} } \cdot 2^{-k}.
\end{equation*}
Recall that we set $\kappa = \exp(-\bar{c})$. For convenience denote by $f(\bar{c})$ the coefficient of $2^{-k}$ in the above expression. By Part~\ref{item:ilc:err=theta} of Lemma~\ref{lem:logconcave}
\begin{equation}\label{eq:tmp:theta_k}
\theta(v_k, w^*) \leq c_2 \err_D(v_k) \leq c_2 f(\bar{c}) \cdot 2^{-k}.
\end{equation}

Now let $g(\bar{c}) = c_2 f(\bar{c}) + 16c_2 \exp(-\bar{c})$. By our choice of $\bar{c}$, $g(\bar{c}) \leq 2^{-8}\pi$. This ensures that for both \eqref{eq:tmp:theta_1} and \eqref{eq:tmp:theta_k}, $\theta(v_k, w^*) \leq 2^{-k-8}\pi$ for any $k \geq 1$.
\end{proof}

\begin{lemma}\label{lem:angle-to-feasible}
For any $1 \leq k \leq K$, if $\theta(v_k, w^*) \leq 2^{-k-8}\pi$, then $w^* \in W_{k+1}$.
\end{lemma}
\begin{proof}
We only need to show that $\twonorm{w_k - w^*} \leq r_{k+1}$. Let $\hat{v}_k = v_k / \twonorm{v_k}$. By algebra $\twonorm{\hat{v}_k - w^*} = 2 \sin\frac{\theta(v_k, w^*)}{2} \leq \theta(v_k, w^*) \leq 2^{-k-8}\pi \leq 2^{-k-6}$. Now we have
\begin{equation*}
\twonorm{w_{k} - w^*} = \twonorm{ \hat{v}_k - w^*} \leq 2^{-k-6} = r_{k+1}.
\end{equation*}
The proof is complete.
\end{proof}

\subsection{Proof of Theorem~\ref{thm:nasty-informal}}\label{sec:app:main-proof}
\begin{proof}
We will prove the theorem with the following claim.

\begin{claim}\label{claim:main-aux}
For any $1 \leq k \leq K$, with probability at least $1 - \sum_{i=1}^{k} \delta_i$, $w^*$ is in $W_{k+1}$.
\end{claim}

Based on the claim, we immediately have that with probability at least $1 - \sum_{k=1}^{K} \delta_k \geq 1 - \delta$, $w^*$ is in $W_{K + 1}$. By our construction of $W_{K + 1}$, we have
\begin{equation*}
\twonorm{w^* - w_{K}} \leq 2^{-K-5}.
\end{equation*}
This, together with Part~\ref{item:ilc:err=theta} of Lemma~\ref{lem:logconcave} and the fact that $\theta(w^*, w_{K}) \leq \pi \twonorm{w^* - w_{K}}$ (see Lemma~10 of \citet{zhang2018efficient}), implies
\begin{equation*}
\err_D(w_{K}) \leq \frac{\pi}{c_1} \cdot 2^{-K-5} = \epsilon.
\end{equation*}

The sample complexity of the algorithm is given by 
\begin{equation*}
N := \sum_{k=1}^{K} N_k  =\sum_{k=1}^{K} \frac{d}{b_k} \cdot \polylog{d, \frac{1}{b_k}, \frac{1}{\delta_k}} \leq \frac{d}{\epsilon} \cdot \polylog{d, \frac{1}{\epsilon}, \frac{1}{\delta}},
\end{equation*}
where we use the fact that $b_k \geq K_1 \epsilon$ for some constant $K_1 > 0$ and $K = O(\log\frac{1}{\epsilon})$.

For each phase $k \leq K$, the number of calls to $\oracley$ equals the size of $T$. For the size of $\TC$, by Lemma~\ref{lem:logconcave} we know that the probability mass of the band $X_k = \{ x: \abs{w_{k-1} \cdot x} \leq b_k \}$ is at most $2b_k$, implying that $\abs{\TC} \leq O( b_k N_k)$ with high probability in view of Chernoff bound. On the other hand, by Part~\ref{item:nasty-TC-size} of Lemma~\ref{lem:nasty-sample-set-size} we have $\abs{\TD} = \abs{\TE} \leq O(b_k N_k)$ since $\xi_k = \Theta(1)$ as indicated in Section~\ref{subsec:param-setting}. Therefore, $\abs{T} \leq O(b_k N_k)$ and the label complexity $m$ of the algorithm is given by
\begin{equation*}
m \leq \sum_{k=1}^{K} b_k N_k = d \cdot \polylog{d, \frac{1}{\epsilon}, \frac{1}{\delta}}.
\end{equation*}

It remains to prove Claim~\ref{claim:main-aux} by induction. First, for $k=1$, $W_1 = \{w: \twonorm{w}\leq 1 \}$. Therefore, $w^* \in W_1$ with probability $1$. Now suppose that Claim~\ref{claim:main-aux} holds for some $k \geq 2$, that is, there is an event $E_{k-1}$ that happens with probability $1 - \sum_{i}^{k-1}\delta_i$, and on this event $w^* \in W_k$. By Lemma~\ref{lem:feasible-to-angle} we know that there is an event $F_k$ that happens with probability $1-\delta_k$, on which $\theta(v_k, w^*) \leq 2^{-k-8} \pi$. This further implies that $w^* \in W_{k+1}$ in view of Lemma~\ref{lem:angle-to-feasible}. Therefore, consider the event $E_{k-1} \cap F_k$, on which $w^* \in W_{k+1}$ with probability $\Pr(E_{k-1}) \cdot \Pr(F_k \mid E_{k-1}) = (1 - \sum_{i}^{k-1}\delta_i) (1 - \delta_k) \geq 1 - \sum_{i=1}^{k} \delta_i$.
\end{proof}

\section{Properties of Isotropic Log-Concave Distributions}\label{sec:logconcave-property}

We record some useful properties of isotropic log-concave distributions.

\begin{lemma}\label{lem:logconcave}
There are absolute constants $c_0, c_1, c_2 > 0$, such that the following holds for all isotropic log-concave distributions $D \in \calD$. Let $f_D$ be the density function. We have
\begin{enumerate}
\item \label{item:ilc:proj} Orthogonal projections of $D$ onto subspaces of $\Rd$ are isotropic log-concave;
\item \label{item:ilc:anti-anti-concen} If $d=1$, then $\Pr_{x \sim D}(a \leq x \leq b) \leq \abs{b-a}$;
\item \label{item:ilc:anti-concen} If $d=1$, then $f_D(x) \geq c_0$ for all $x \in [-1/9, 1/9]$;
\item \label{item:ilc:err=theta} For any two vectors $u, v \in \Rd$,
\begin{equation*}
c_1 \cdot \Pr_{x \sim D}\del{\sign{u\cdot x} \neq \sign{v \cdot x}} \leq \theta(u, v) \leq c_2 \cdot \Pr_{x \sim D}\del{\sign{u\cdot x} \neq \sign{v \cdot x}};
\end{equation*}
\item \label{item:ilc:tail} $\Pr_{x \sim D}\big( \twonorm{x} \geq t \sqrt{d} \big) \leq \exp(-t + 1)$.
\end{enumerate}
\end{lemma}

We remark that Parts~\ref{item:ilc:proj},~\ref{item:ilc:anti-anti-concen},~\ref{item:ilc:anti-concen}, and~\ref{item:ilc:tail} are due to \citet{lovasz2007geometry}, and Part~\ref{item:ilc:err=theta} is from \citet{vempala2010random,balcan2013active}. 

The following lemma is implied by the proof of Theorem~21 of \citet{balcan2013active}, which shows that if we choose a proper band width $b > 0$, the error outside the band will be small. This observation is crucial for controlling the error over the distribution $D$, and has been broadly recognized in the literature~\cite{awasthi2017power,zhang2018efficient}.

\begin{lemma}[Theorem~21 of \citet{balcan2013active}]\label{lem:err outside band}
There are absolute constants $c_3, c_4 > 0$ such that the following holds for all isotropic log-concave distributions $D \in \calD$. Let $u$ and $v$ be two unit vectors in $\Rd$ and assume that $\theta(u, v) = \alpha < \pi/2$. Then for any $b \geq \frac{4}{c_4} \alpha$, we have
\begin{equation*}
\Pr_{x \sim D}(\sign{u\cdot x} \neq \sign{v \cdot x}\ \text{and}\ \abs{v\cdot x} \geq b) \leq c_3 \alpha \exp\(- \frac{c_4 b}{2 \alpha}\).
\end{equation*}
\end{lemma}

\begin{lemma}\label{lem:x-2norm}
Suppose $x$ is randomly drawn from $D_{u, b}$. Then with probability $1-\delta$, $\twonorm{x} \leq c_7 \sqrt{d} \log\frac{1}{b\delta}$ for some constant $c_7 > 0$.
\end{lemma}
\begin{proof}
Using Part~\ref{item:prob-band:D-to-band} of Lemma~\ref{lem:P(x in band)}, we have
\begin{equation*}
\Pr_{x \sim D_{u, b}}( \twonorm{x} \geq \alpha) \leq \frac{1}{c_8 b} \Pr_{x \sim D}( \twonorm{x} \geq \alpha) \leq \frac{e}{c_8 b} \exp\del{- \alpha / \sqrt{d}},
\end{equation*}
where we applied Part~\ref{item:ilc:tail} of Lemma~\ref{lem:logconcave} in the last inequality. The lemma follows by setting the right-hand side to $\delta$.
\end{proof}

\begin{lemma}\label{lem:P(x in band)}
Let $c_8 =  \min\big\{2c_0, \frac{2c_0}{9 C_1}, \frac{1}{C_1}\big\}$. Then for all isotropic log-concave distributions $D \in \calD$,
\begin{enumerate}

\item \label{item:prob-band:refined-lower}  $\Pr_{x \sim D}\( \abs{u \cdot x} \leq b \) \geq c_8 \cdot b$;

\item \label{item:prob-band:D-to-band} $\Pr_{x \sim D_{u, b}}(E ) \leq \frac{1}{c_8 b} \Pr_{x \sim D}(E)$ for any  event $E$.
\end{enumerate}
\end{lemma}
\begin{proof}
We first consider the case that $u$ is a unit vector.

For the lower bound, Part~\ref{item:ilc:anti-concen} of Lemma~\ref{lem:logconcave}  shows that the density function of the random variable $u\cdot x$ is lower bounded by $c_0$ when $\abs{u \cdot x} \leq 1/9$. Thus
\begin{align*}
\Pr_{x \sim D}\( \abs{u \cdot x} \leq b\) \geq \Pr_{x \sim D}\( \abs{u \cdot x} \leq \min\{b, {1}/{9}\}\) \geq 2 c_0 \min\{b, {1}/{9}\} \geq 2c_0 \min\bigg\{1, \frac{1}{9 C_1}\bigg\} \cdot b
\end{align*}
where in the last inequality we use the condition $b \leq C_1$.

For any event $E$, we always have
\begin{equation*}
\Pr_{x \sim D_{u, b}}(E) \leq \frac{\Pr_{x \sim D}(E)}{\Pr_{x \sim D}({\abs{u\cdot x}\leq b})} \leq \frac{1}{c_8 b} \Pr_{x \sim D}(E).
\end{equation*}

Now we consider the case that $u$ is the zero vector and $b = C_1$. Then $\Pr_{x \sim D}\( \abs{u \cdot x} \leq b \) = 1 \geq c_8 \cdot b$ in view of the choice $c_8$. Thus Part~\ref{item:prob-band:D-to-band} still follows. The proof is complete.
\end{proof}

\end{document}